\documentclass{new_tlp}
\usepackage{amssymb}
\usepackage{amsmath}
\usepackage{url}
\usepackage{color}

\usepackage{IEEEtrantools}
\usepackage{enumitem}
\usepackage{krudces-main}
\usepackage{krudces-thm}
\usepackage{krudces-thme}
\def\Not{\hbox{\rm \em not }}

\def\u{\hbox{\bf u}}
\def\P{{\cal P}}
\def\C{{\cal C}}
\def\F{{\cal F}}
\def\T{{\cal T}}
\def\cS{\mathcal{S}}
\def\cSB{\mathcal{SB}}
\def\h{h}
\def\t{t}
\def\SQHT{{\hbox{\rm SQHT}^=}}
\def\SQHTF{{\hbox{\rm SQHT}^=_\F}}
\def\SQHTS{{\hbox{\rm SQHT}^=_\cS}}
\def\QELF{{\hbox{\rm QEL}^=_\F}}

\def\qed{~\hfill$\Box$}

\def\signature{\Sigma = \tuple{\C,\F,\P}}
\def\Terms{\mathit{Terms}}
\def\TermsF{\mathit{Terms_\F}}
\def\HBF{\ensuremath{\mathcal{HB}_\F}}
\def\HB{\ensuremath{\mathcal{HB}}}
\def\Uni{\ensuremath{\mathcal{D}}}

\def\IAA{\ensuremath{\IA = \tuple{\sigma^h,\sigma^t,I^h,I^t}}}

\def\IAAa{\ensuremath{\IA_1 = \tuple{\sigma^h_1,\sigma^t_1,I^h_1,I^t_1}}}
\def\IAAb{\ensuremath{\IA_2 = \tuple{\sigma^h_2,\sigma^t_2,I^h_2,I^t_2}}}
\def\IAH{\ensuremath{\IA = \tuple{\sigma^h,\sigma^t,H,T}}}

\def\modelsF{\models_{\hspace{-1pt}\scriptscriptstyle\F}}
\def\modelsS{\models_{\hspace{-1pt}\scriptscriptstyle\cS}}

\def\Alog{{\cal A}\mathit{log}}

\newcommand\Coherent[1]{\mathit{Coh}(#1)}

\newcommand\GZatom{\mbox{GZ-atom}}
\newcommand\GZformula{\mbox{GZ-formula}}
\newcommand\GZtheory{\mbox{GZ-theory}}

\def\modelscl{\models_{\!\mathit{cl}}}
\def\fF{\varphi}
\def\fG{\psi}

\def\IAT{\ensuremath{\IA = \tuple{\sigma,T}}}

\def\grnd{\mathtt{Gr}}
\newcommand{\grndp}[2]{\mathtt{Gr}^+_{#1}(#2)}

\usepackage{krudces-notation}

\definecolor{darkred}{rgb}{0.8,0,0.2}

\def\rel{\unlhd}

\begin{document}
\bibliographystyle{acmtrans}

\submitted{21 June 2009}
\revised{}
\accepted{}
\title[Functional ASP with Intensional Sets]
    {Functional ASP with Intensional Sets:\\
     Application to Gelfond-Zhang Aggregates\thanks{Partially supported by MINECO, Spain, grant TIC2017-84453-P, Xunta de Galicia, Spain (GPC ED431B 2016/035 and 2016-2019 ED431G/01, CITIC). The second author is funded by the Centre International de Math\'{e}matiques et d'Informatique de Toulouse (CIMI) through contract ANR-11-LABEX-0040-CIMI within the program ANR-11-IDEX-0002-02. The fourth author is supported by UPM project RP151046021.}}

  \author[P. Cabalar, J. Fandinno, L. Fari{\~n}as and D. Pearce]
         {Pedro Cabalar, Jorge Fandinno, Luis Fari{\~n}as del Cerro and David Pearce\\
          Department of Computer Science,
          University of Corunna, Corunna, Spain.\\
          {\email{cabalar@udc.es}}\\\\
          IRIT, Universit{\'e} de Toulouse, CNRS, Toulouse, France\\
          {\email{\{jorge.fandinno, luis\}@irit.fr}}\\\\
          Universidad Polit\'ecnica de Madrid, Madrid, Spain\\
      {\email{david.pearce@upm.es}}}


\maketitle

\begin{abstract}
In this paper, we propose a variant of Answer Set Programming (ASP) with evaluable functions that extends their application to sets of objects, something that allows a fully logical treatment of aggregates.
Formally, we start from the syntax of First Order Logic with equality and the semantics of Quantified Equilibrium Logic with evaluable functions ($\QELF$).
Then, we proceed to incorporate a new kind of logical term, \emph{intensional set} (a construct commonly used to denote the set of objects characterised by a given formula), and to extend $\QELF$ semantics for this new type of expression.
In our extended approach, intensional sets can be arbitrarily used as predicate or function arguments or even nested inside other intensional sets, just as regular first-order logical terms.
As a result, aggregates can be naturally formed by the application of some evaluable function ({\tt count}, {\tt sum}, {\tt maximum}, etc) to a set of objects expressed as an intensional set.
This approach has several advantages.
First, while other semantics for aggregates depend on some syntactic transformation (either via a reduct or a formula translation), the $\QELF$ interpretation treats them as regular evaluable functions, providing a compositional semantics and avoiding any kind of syntactic restriction.
Second, aggregates can be explicitly defined now within the logical language by the simple addition of formulas that fix their meaning in terms of multiple applications of some (commutative and associative) binary operation.
For instance, we can use recursive rules to define {\tt sum} in terms of integer addition.
Last, but not least, we prove that the semantics we obtain for aggregates coincides with the one defined by Gelfond and Zhang for the $\Alog$ language, when we restrict to that syntactic fragment.\\
(\emph{Under consideration for acceptance in TPLP})
\end{abstract}

\section{Introduction}

\label{sec:intro}

Due to its extensive use for practical Knowledge Representation and Reasoning (KRR), the paradigm of \emph{Answer Set Programming} (ASP;\ \citeNP{baral2003knowledge}) has been continuously subject to multiple extensions of its input language and, frequently, its formal semantics.
One of those possible extensions is the addition of \emph{evaluable functions} (see~\citeNP{Cab13} for a survey).
This extension allows us, for instance, to replace the conjunction $mother(cain,X)\wedge mother(abel,X)$ by the equality $mother(cain)$ = $mother(abel)$ so that (i) $mother$ can be better captured as a function (a person has a unique mother) and (ii) it is not treated as a Herbrand function, since syntactically different terms may refer to the same object.
Although several prototypes for functional ASP have been developed~\cite{LW08,Cabalar11,Balduccini13,BL14}, the use of evaluable functions has not been commonly adopted in the mainstream ASP solvers yet.
Still, their logical definition can also be useful for other common ASP extensions, as happened with their application to constraint ASP~\cite{CKOS16}.
Another ASP extension that can be examined under the functional viewpoint is the use of aggregates.
An \emph{aggregate} is the result of an operation on a set of values, such as their cardinality, their sum, their maximum/minimum value, etc.
ASP introduces this feature via so-called \emph{aggregate atoms}, that allow comparing the result of an aggregate with some fixed value.
To put an example, suppose $p(X)$ means that Agatha Christie wrote book $X$.
Then, adding the aggregate atom $\mathtt{count}\set{X:p(X) } \geq n$ in a rule body checks that she wrote at least $n$ books.
Defining the semantics for these atoms may become tricky, since it is easy to build self-referential rules like:
\begin{eqnarray}
p(a) \leftarrow \mathtt{count}\set{X:p(X) }\geq n.
\label{f:count1}
\end{eqnarray}
to express that Mrs. Christie also writes an autobiography $a$ if she writes at least $n$ books.
Different alternative semantics have been proposed for ASP aggregates~\cite{siniso02a,pedebr07a,sopo07a,ferraris11a,fapfle11a,gezh14a} but all of them have treated each aggregate atom as a whole, without providing a semantics for its individual components.
A different, and perhaps more natural possibility, is to consider inequality as a standard predicate and interpret $\mathtt{count}(S)$ as an \emph{evaluable function}, whose argument $S$ happens to be a set.

In this paper, we propose an extension of ASP with evaluable functions that allows their application to sets of objects and the treatment of aggregates as functions.
To this aim, we start from the first-order logic characterisation of ASP, \emph{Quantified Equilibrium Logic} (QEL;\ \citeNP{PV04}) plus its extension to evaluable functions ($\QELF$; \citeNP{Cabalar11}).
Then, we proceed to include a new type of logical term, \emph{intensional set}, with the form $\set{ \vec{\tau}(\vec{x}) : \varphi(\vec{x}) }$ and the expected meaning, that is, the set of tuples $\vec{\tau}(\vec{c})$ for which the formula $\varphi(\vec{c})$ holds  
(having $\vec{c}$ and $\vec{x}$ the same arity).
Intensional sets constitute a quite common mathematical notation and, in fact, have been already studied in the context of Prolog~\cite{DovierOPR91} and Constraint Logic Programming~\cite{DovierPR03}.
In our case, we will treat them as regular, first-order logical terms, without syntactic restrictions, so they can be arbitrarily nested in other expressions. 
One interesting feature inherited from $\QELF$ is that functions can be partial, so we can use them to represent that, say, $mother(adam)$, $mother(eve)$ or $division(3,0)$ are undefined, but also that the maximum value of an empty set $\mathtt{max}(\emptyset)$ is undefined too.
The new extension allows now to define new aggregates within the logical language.
It suffices 
to add
formulas to fix their meaning in terms of multiple applications of some (commutative and associative) binary operation.
For instance, we show how to define the {\tt sum} aggregate using recursive rules in terms of 
integer addition.
Finally, we are also able to prove that, when restricted to the the $\Alog$ language~\cite{gezh14a}, there is a semantic, one-to-one correspondence.

The rest of the paper is organised as follows.
In Section~\ref{sec:eq} we recall the basic definitions of Functional ASP under the $\QELF$ interpretation.
Section~\ref{sec:isets} introduces intensional sets while Section~\ref{sec:agg} studies their use for aggregates.
Section~\ref{sec:gz} focuses on the correspondence to \cite{gezh14a} aggregates.
Finally, Section~\ref{sec:conc} concludes the paper.

\section{Background: Quantified Equilibrium Logic with Evaluable Functions}
\label{sec:eq}

The definition of propositional Equilibrium Logic~\cite{Pea96} relied on a selection criterion on models of the intermediate logic of \emph{Here-and-There}~(HT;~\citeNP{Hey30}). The first order case~\mbox{\cite{PV04}} followed similar steps, introducing a quantified version of HT, called $\SQHT$ that stands for \emph{Quantified HT with static domains\footnote{The term \emph{static domain} means that the universe is shared among all worlds in the Kripke frame.} and equality}.  In this section we describe the syntax and semantics of a variant of the latter, called $\SQHTF$~\cite{Cabalar11}, for dealing with evaluable functions. 

We begin by defining a first-order language by its \emph{signature}, a tuple $\signature$ of disjoint sets where $\C$ and $\F$ are sets of \emph{function names} and $\P$ a set of \emph{predicate names}. We assume that each function (resp. predicate) name has the form $f/n$ where $f$ is the function (resp. predicate) symbol, and $n\geq 0$ is an integer denoting the number of arguments (or \emph{arity}). Elements in $\C$ will be called \emph{Herbrand functions} (or \emph{constructors}), whereas elements in $\F$ will receive the name of \emph{evaluable functions} (or \emph{operations}). The sets $\C_0$ (Herbrand constants) and $\F_0$ (evaluable constants) respectively represent the elements of $\C$ and $\F$ with arity $0$. We assume $\C_0$ contains at least one element.

First-order formulas follow the syntax of classical predicate calculus with equality~``$=$''.
We assume that default negation~\mbox{$\neg \varphi$} is defined as \mbox{$\varphi \rightarrow \bot$}.
We use letters $x, y, z$ and their capital versions to denote variables, $\tau$ to denote terms, and letters $c, d, e$ to denote ground terms.
Tuples of variables, terms and ground terms are respectively represented by $\vec{x}, \vec{\tau}, \vec{c}$.
By abuse of notation, when a tuple contains a single element, we write just $\tau$ instead of $\tuple{\tau}$.
When writing formulas, we assume that all free variables are implicitly universally quantified. 
An atom like \mbox{$\tau=\tau'$} is called an \emph{equality atom}, whereas an atom like \mbox{$p(\tau_1,\dots,\tau_n)$} for any predicate $p/n$ different from equality receives the name of \emph{predicate atom}.
Given any set of functions $S$ we write $\TermsF(S)$ to stand for the set of ground terms built from functions (and constants) in~$S$. In particular, the set of all possible ground terms for the signature $\signature$ would be $\TermsF(\C \cup \F)$ whereas the subset $\TermsF(\C)$ will be called the \emph{Herbrand Universe} of the language ${\cal L}_{\F}(\Sigma)$. The \emph{Herbrand Base}~$\HBF(\C,\P)$ is the set containing all atoms that can be formed with predicates in $\P$ and terms in the Herbrand Universe, $\TermsF(\C)$.


\begin{definition}[$\SQHTF$-assignment]
An $\SQHTF$-\emph{assignment} $\sigma$ for a signature \mbox{$\tuple{\C,\F,\P}$}
is a function
\mbox{$\sigma : \TermsF(\C \cup \F) \rightarrow \TermsF(\C) \cup \{ \u \}$}
that maps any ground term in the language to some ground term in the Herbrand Universe or the special value $\u \not\in \TermsF(\C \cup \F)$ (standing for \emph{undefined}). The function $\sigma$ must satisfy:
\begin{enumerate}[ label=(\roman*), leftmargin=20pt]
\item $\sigma(c) \eqdef c$ for all $c \in \TermsF(\C)$.
\item 
$\sigma(f(\tau_1,\dots,\tau_n)) \eqdef \left\{
\begin{array}{ll}
\u & \hbox{if } \sigma(\tau_i)=\u \ \hbox{for some } i=1\dots n \\
\sigma(f(\sigma(\tau_1),\dots,\sigma(\tau_n))) & \hbox{otherwise} 
\end{array}
\right.$
\end{enumerate}
\vspace{-15pt}
\qed
\end{definition}

As we can see, the value of any functional term is an element from the Herbrand Universe $\TermsF(\C)$, excluding the cases in which operations are left undefined (i.e., they are \emph{partial} functions) -- if so, they are assigned the special element $\u$ (outside the universe) instead.
Condition (i) asserts, as expected, that any term $c$ from the Herbrand Universe has the fixed valuation $\sigma(c)=c$. Condition (ii) asserts that a functional term with an undefined argument becomes undefined in its turn (functions like these are called \emph{strict}). Otherwise, if all arguments are defined, then functions preserve their interpretation through subterms -- for instance, if we have $\sigma(f(a))=c$ we expect that $\sigma(g(f(a)))$ and $\sigma(g(c))$ coincide. It is easy to see that (ii) implies that $\sigma$ is completely determined by the mappings $f(\vec{c})=d$ where $f$ is any operation, $\vec{c}$ a tuple of elements from $\TermsF(\C)$, and $d$ an element in the latter. We call these expressions \emph{ground functional facts}.

\begin{definition}[Ordering $\preceq$ among assignments]
\label{def:sqht.order}
Given two assignments $\sigma, \sigma'$ we define $\sigma \preceq \sigma'$ as the condition: 
\mbox{$\sigma(\tau)=\sigma'(\tau)$} or \mbox{$\sigma(\tau)=\u$}, for all terms $\tau \in \TermsF(\C \cup \F)$.
\qed
\end{definition}
As usual, we write \mbox{$\sigma \prec \sigma'$} when \mbox{$\sigma \preceq \sigma'$} and \mbox{$\sigma \neq \sigma'$}.
The intuitive meaning of
\mbox{$\sigma \preceq \sigma'$}
is that both contain \emph{compatible information}, but the former contains \emph{less information} than the latter: any defined function in
$\sigma$ must preserve the same value in $\sigma'$.

\begin{definition}[$\SQHTF$-interpretation]
\label{def:sqht-interp}
An {\normalfont$\SQHTF$}-\emph{interpretation} $\IA$ for a signature \mbox{$\signature$} is a
quadruple $\tuple{\sigma^h\!,\sigma^t\!,I^h\!,I^t}$
where
\mbox{$I^h \subseteq I^t \subseteq \HBF$}
are two sets of ground atoms
and $\sigma^h$ and  $\sigma^t$
are two assigments
satisfying
\mbox{$\sigma^h \preceq \sigma^t$}.
\qed
\end{definition}

The superindices $\h,\t$ represent two intuitionistic Kripke worlds (respectively standing for \emph{here} and \emph{there}) with a reflexive ordering relation satisfying $\h \leq \t$. 
Accordingly, world $h$ contains less information than $t$, as we can see in the conditions $I^h \subseteq I^t$ and $\sigma^h \preceq \sigma^t$.
We say that the interpretation $\IA$ is \emph{total}\footnote{Note that by \emph{total} we do not mean that functions cannot be left undefined. We may still have some term $d$ for which $\sigma^\h(d)=\sigma^\t(d)=\u$.} when both worlds contain the same information, that is, $I^h = I^t$ and $\sigma^h = \sigma^t$, and we abbreviate it as the pair $\tuple{\sigma^t,I^t}$. 
Moreover, given $\IAA$ we define its corresponding total interpretation $\IA^t$ as $\tuple{\sigma^t,I^t}$ that is, the one in which all the uncertainty in world $h$ is ``filled'' with the information in $t$. 

An interpretation $\IAA$ \emph{satisfies} a formula $\varphi$ at some world $w\in \{h,t\}$, written $\IA,w \modelsF \varphi$, when any of the following conditions are satisfied:

\begin{enumerate}[ label=\roman*)]
\item $\IA,w \modelsF p(\tau_1,\dots,\tau_n)$ if $p( \sigma^w(\tau_1), \dots, \sigma^w(\tau_n) ) \in I^w$ for any predicate $p/n \in {\cal P}$;
\item $\IA,w \modelsF \tau_1=\tau_2$ if $\sigma^w(\tau_1)=\sigma^w(\tau_2) \neq \u$;
\item $\IA,w \not\modelsF \bot$; $\IA,w \modelsF \top$;
\item $\IA,w \modelsF \alpha \wedge \beta$ if $\IA,w \modelsF \alpha$ and $\IA,w \modelsF \beta$;
\item $\IA,w \modelsF \alpha \vee \beta$ if $\IA,w \modelsF \alpha$ or $\IA,w \modelsF \beta$;
\item $\IA,w \modelsF \alpha \rightarrow \beta$ if for all $w'\geq w$: $\IA,w' \not\modelsF \alpha$ or $\IA,w' \modelsF \beta$;
\item $\IA,w \modelsF \forall x \ \alpha(x)$ if for each
$c \in \TermsF(\C)$:
$\IA,w \modelsF \alpha(c)$;
\item $\IA,w \modelsF \exists x \ \alpha(x)$ if for some $c \in \TermsF(\C)$: $\IA,w \modelsF \alpha(c)$.
\end{enumerate}
The first condition above implies that an atom with an undefined argument will always be evaluated as false since, by definition, $\u$ never occurs in ground atoms of $I^\h$ or $I^\t$. Something similar happens with equality:
\mbox{$\tau_1 = \tau_2$}
will be false if any of the two operands, or even both, are undefined. As usual, $I$ is called a \emph{model} of a theory $\Gamma$, written $I \modelsF \Gamma$, when $I,h \modelsF \varphi$ for all $\varphi \in \Gamma$.

\begin{proposition}[From~\protect\citeNP{Cabalar11}]\label{prop:SQHTF.neg}
$\IA,h \modelsF \neg \varphi$ \ $\Leftrightarrow$\ \ $\IA,t \modelsF \neg \varphi$\ $\Leftrightarrow$\ \ $\IA,t \not\modelsF \varphi$.\qed
\end{proposition}

We define next a particular ordering among $\SQHTF$-interpretations.
We say that $\IAAa$ is smaller than $\IAAb$, also written $\IA_1 \preceq \IA_2$ by abuse of notation, when $I^t_1=I^t_2$, $\sigma^t_1=\sigma^t_2$, $I^h_1 \subseteq I^h_2$ and $\sigma^h_1 \preceq \sigma^h_2$.
That is, they have the same information at world $t$, but $\IA_1$ can have less information
than $\IA_2$ at world $h$.
Again, we write $\IA_1 \prec \IA_2$ when $\IA_1 \preceq \IA_2$ and $\IA_1 \neq \IA_2$.
Nonmonotonicity is obtained by the next definition, which introduces the idea of equilibrium models for $\SQHTF$.

\begin{definition}[Equilibrium model]\label{def:SQHTF.equilibrium}
A total model $\IA=\tuple{\sigma,I}$ of a theory $\Gamma$ is an \emph{equilibrium model} if there is no strictly smaller interpretation $\IA' \prec \IA$ that is also a model of $\Gamma$.
A set of atoms $I$ is a \emph{stable model}\footnote{Apart from atoms, we could additionally include ground functional facts $f(\vec{c})=d$. However, we only consider atoms here, for better comparison to other (non-functional) semantics of aggregates.} of $\Gamma$ iff $\tuple{\sigma,I}$ is an equilibrium model for some $\sigma$.\qed
\end{definition}

\section{QEL with Evaluable Functions and Intensional Sets}
\label{sec:isets}

In this section, we define $\SQHTS$: a logic that extends $\SQHTF$ with \emph{intensional sets}.
Intensional sets are terms of the form 
$\set{ \vec{x} : \vec{\tau}(\vec{x}) : \varphi(\vec{x}) }$
where 
$\vec{x}$ is a tuple of variables and
$\vec{\tau}(\vec{x})$ and $\varphi(\vec{x})$ are respectively a tuple of terms and a formula with free variables~$\vec{x}$.

Note that, as opposed to terms in $\SQHTF$ (and also in first order logic), intensional sets are terms whose structure not only depends on other terms, but also on formulas.
Hence, we define terms and formulas recursively such that
the definition of $i$-terms will depend on the definition of $(i-1)$-formulas
while the definition of $i$-formulas will depend on the definition of $i$-terms.
We depart from a similar first-order signature~$\signature$ as in $\SQHTF$, but we build the terms as follows.
For any $i\geq 0$, we define an \emph{$i$-term} as any of the following cases:
\begin{enumerate}[ topsep=2pt, itemsep=0pt, label=\roman*), start=1 ]

\item every constant $c \in \C_0 \cup \F_0$.

\item every variable $x$.

\item $f(\tau_1,\dots,\tau_n)$, where \mbox{$f/n \in (\C \cup \F)$} 
and $\tau_1,\dots,\tau_n$ are $i$-terms, in their turn.

\item the construct $\set{\vec{\tau}_1, \dotsc, \vec{\tau}_{m}}$ (called \emph{extensional set}), where $m\geq 0$ and $\vec{\tau}_1, \dotsc, \vec{\tau}_m$ are $n$-tuples (of the same arity $n\geq 1$) of $i$-terms .
If $m=0$ we write $\emptyset$ instead of $\{\}$.

\item the construct $\set{\vec{\tau} \!:\! \varphi }$ (called \emph{intensional set}) if $i>0$, $\varphi$ is an $(i-1)$-formula and $\vec\tau$ is a tuple of $i$-terms.
\label{item:set.definition.term}
\end{enumerate}
Now, for any $i\geq0$, $i$-\emph{atoms} and $i$-\emph{formulas} are defined over $i$-terms as follows:
\begin{enumerate}[ topsep=2pt, itemsep=0pt, label=\roman*), start=6 ]

\item 
if $\tau_1,\dotsc, \tau_n$ are $i$-terms\ and $p/n \in {\cal P}$, then $p(\tau_1,\dotsc, \tau_n)$ is an $i$-atom,

\item if $\tau_1$ and $\tau_2$ are $i$-terms,
then \mbox{$\tau_1 = \tau_2$} is an $i$-atom\

\item every $i$-atom\ is an $i$-formula,

\item $\bot$ and $\top$ are $0$-formulas,

\item if $\varphi_1$ and $\varphi_2$ are $i$-formulas, then $\varphi_1 \otimes \varphi_2$ is an $i$-formula\ with $\otimes\in \set{\wedge,\vee,\to}$.

\item if $\varphi$ is an $i$-formula\ and $x$ is a variable, then $\forall x \, \varphi$ and $\exists x \, \varphi$ are $i$-formulas.
\end{enumerate}
A \emph{formula} (resp. \emph{term}) is any $i$-formula (resp. $i$-term) for any $i\geq 0$.
Note that \ref{item:set.definition.term} is the only case in the term definition that refers to a formula, but this formula has less rank than the term, so the definition is well-founded.
$\Terms^i(\C \cup \F)$ denotes all the ground $i$-terms
while 
$\Terms^i(\C)$ denotes all ground $i$-terms without evaluable functions.
$\Terms(\C \cup \F) = \bigcup_{0 \leq i}\Terms^i(\C \cup \F)$
and
$\Terms(\C) = \bigcup_{0 \leq i}\Terms^i(\C)$
denote the set of all ground terms and ground terms without evaluable functions.
In particular,
$\Terms^0(\C)$ corresponds to the \emph{Herbrand Universe} that includes not only $\TermsF(\C)$ we had before, but also all possible formations of extensional sets, that act as a new constructor.
For instance, if we have the singleton $\C=\{c\}$, then $\TermsF(\C)=\{c\}$ is finite but $\Uni$ additionally contains an infinite number of (finite) extensional sets including, among others, the sets of tuples $\emptyset, \{c\}$\footnote{Recall that, here, $c$ stands for the unary tuple $\tuple{c}$.}, $\{\tuple{c,c}\}$, $\{\tuple{c,c,c}\}$, \dots, or combinations of nested sets such as $\{\{c\}\}$ or $\{\{c\}, \{\tuple{c,c}\}\}$, etc.
We also define
$\cSB \eqdef \Terms(\C) \setminus \TermsF(\C)$,
that is, the subset of the Herbrand universe consisting of extensional sets.
In the previous example $\cSB = \Terms(\C) \setminus \{c\}$.
By $\HB$ we denote the \emph{Herbrand Base}, that is, the set of all ground atoms
of the form $p(c_1,\dotsc,c_n)$ with $p/n \in {\cal P}$
and $\set{c_1,\dotsc,c_n} \subseteq \Terms(\C)$.

If we consider terms also formed with evaluable functions, we have that$\TermsF(\C \cup \F) \subseteq \Terms^0(\C \cup \F)$ again,
and so, every $\SQHTF$ formula is also a $\SQHTS$ formula --
obviously the converse does not hold, as the latter may contain set constructors
such as $p(\{c\})$.
Still, we could take each possible extensional set in $\cSB$ as a kind of Herbrand constant like those in $\C$.
Doing so, \mbox{$\TermsF(\C \cup \cSB \cup \F) = \Uni$} and, thus, every \mbox{$0$-formula} over a signature~\mbox{$\signature$} 
is also a $\SQHTF$ formula over the signature
\mbox{$\Sigma' = \tuple{\C \cup \cSB,\F,\P}$}.
Note, however, that intensional sets are syntactically different from any $\SQHTF$ term, so there are $\SQHTS$ terms (resp. formulas) that are not $\SQHTF$ terms (resp. formulas) over any signature.

For any expression $\alpha$ (term or formula), we define next when an occurrence of a variable is either \emph{free} or \emph{bound} to some quantifier/intensional set:
\begin{enumerate}[ topsep=2pt, itemsep=0pt, label=\roman*) ]
\item all free occurrences of $x$ in $\psi$ are bound in $\exists x \, \psi$ to its prefix quantifier $\exists x$
\item all free occurrences of $x$ in $\psi$ are bound in $\forall x \, \psi$ to its prefix quantifier $\forall x$
\item if $x$ occurs in $\vec{x}$, then all free occurrences of $x$ in $\psi$ and $\vec\tau$ are bound in $\set{\vec{x} \!:\! \vec{\tau}\!:\!\psi}$ to the outermost intensional set.

\item In the remaining cases, an occurrence of $x$ is bound (to some connective) in a formula iff it is so in some subformula; otherwise, it is free.
\end{enumerate}

As in the case of $\SQHTF$, when we write standalone formulas, we assume that all free variables are implicitly universally quantified.
Similarly,
we write
$\set{\vec{\tau}\!:\!\psi}$
instead of 
\mbox{$\set{\vec{x} \!:\! \vec{\tau}\!:\!\psi}$}
when $\vec{x}$ contains exactly all free variables in
$\vec{\tau}$.
Note that intensional sets play a role similar to quantifiers.
As an example, suppose we want to obtain the maximum number of times that the character $Poirot$ is mentioned in an Agatha Christie book $b$, and assume that predicate $\mathit{word}(b,i,w)$ tells us that the $i$-th word of book $b$ is $w$.
We assume by now that we have functions $\mathtt{count}$ and $\mathtt{max}$ on sets: their meaning will be fixed later on.
The set $\{i: word(b,i,Poirot)\}$ collects all occurrences of word $Poirot$ in book $b$.
Since $i$ is the only free variable to the left of `$:$', the intensional set is an abbreviation of $\{i: i: word(b,i,Poirot)\}$ revealing that $i$ is being varied. 
On the contrary, variable $b$ is left free.
Now, take the expression 
\begin{eqnarray}
\mathtt{max}\{\mathtt{count}\{i: word(b,i,Poirot)\}: author(Agatha,b)\} \label{f:poirot} 
\end{eqnarray}
The left term $\mathtt{count}\{i: word(b,i,Poirot)\}$ contains a free occurrence of $b$ while $i$ is bound to the inner intensional set.
Therefore, \eqref{f:poirot} actually stands for:
$$\mathtt{max}\{b: \mathtt{count}\{i: i: word(b,i,Poirot)\}: author(Agatha,b)\}$$
that is obviously less readable than \eqref{f:poirot}.
However, in the general case, we may need to make use of the explicit list of quantified variables.
For instance, if we want to parameterise the expression above for some author $x$ and character $y$ whose values are determined outside the term (as part of a formula), then we would necessarily write:
\begin{eqnarray}
\mathtt{max}\{b: \mathtt{count}\{i: word(b,i,y)\}: author(x,b)\} \label{f:poirot2} 
\end{eqnarray}
because the free occurrence of $y$ in $\mathtt{count}\{i: word(b,i,y)\}$ could make us incorrectly assume that it is being varied in the set, as happened with $b$.



\subsection{Semantics}

First, we need to define the domain in which terms are going to be interpreted.
Given a set $S$, let us define the set of all possible $n$-tuples of elements from $S$, for any $n \geq 1$, as 
$$\mathit{Tup}(S) \ \eqdef \ \bigcup_{n\geq 1} \setm{ \vec{e} }{ \vec{e} \in S^n }$$
Our domain $\Uni$ is inductively constructed as follows:
$$
\Uni^0 \ \eqdef \ \TermsF(\C) \hspace{50pt}
\Uni^{i+1} \ \eqdef \ \Uni^i \cup 2^{\mathit{Tup}(\Uni^i)}
$$
so that $\Uni \eqdef \bigcup_{0 \leq i} \Uni^i$.
We also define the subset of $\Uni$ consisting of sets as $\cS \eqdef \Uni \backslash \Uni^0$.

Definitions of assignments and interpretations are then straightforward:
we just extend the domain of $\sigma$ from $\TermsF(\C \cup \F)$ to $\Terms(\C \cup \F)$ and replace the Herbrand Universe $\TermsF(\C)$ by its corresponding $\Uni$. 
\begin{definition}[Assignment]\label{def:assign}
An \emph{assignment} $\sigma$ for a signature \mbox{$\signature$}
is a function
\mbox{$\sigma : \Terms(\C \cup \F) \rightarrow \Uni \cup \{ \u \}$}
that maps some ground term in the Herbrand Universe or the special value
\mbox{$\u \not\in Terms(\C \cup \F)$} (standing for \emph{undefined}) to any ground term in the language. 
Function $\sigma$ must satisfy:
\begin{enumerate}[ label=(\roman*), leftmargin=20pt]
\item $\sigma(c) \eqdef c$ for all $c \in \Uni$.
\item 
$\sigma(f(\tau_1,\dots,\tau_n)) \eqdef \left\{
\begin{array}{ll}
\u & \hbox{if } \sigma(\tau_i)=\u \ \hbox{for some } i=1\dots n \\
\sigma(f(\sigma(\tau_1),\dots,\sigma(\tau_n))) & \hbox{otherwise} 
\end{array}
\right.$
\item \label{def:assign3}
$\sigma(\{\vec\tau_1,\dots,\vec\tau_n\}) \eqdef \left\{
\begin{array}{ll}
\u & \hbox{if } \sigma(\vec\tau_i)=\u \ \hbox{for some } i=1\dots n\\
\{\, \sigma(\vec\tau_1),\dots,\sigma(\vec\tau_n)\, \} & \hbox{otherwise} 
\end{array}
\right.$\\
where\\
$\sigma(\tuple{\tau_1,\dots,\tau_m}) \eqdef \left\{
\begin{array}{ll}
\u & \hbox{if } \sigma(\tau_j)=\u \ \hbox{for some } j=1\dots m\\
\tuple{\, \sigma(\tau_1),\dots,\sigma(\tau_m)\, } & \hbox{otherwise} 
\end{array}
\right.$
\end{enumerate}
\vspace{-15pt}
\qed
\end{definition}
\noindent
Note that we added a third case (iii) for extensional sets, but there is no restriction on the values of intensional sets: their meaning will be fixed later, once we describe the satisfaction of formulas.

\begin{definition}[Ordering $\preceq$ among assignments]
\label{def:order}
Given two assignments $\sigma, \sigma'$ we define $\sigma \preceq \sigma'$, as the condition:
\mbox{$\sigma(\tau)=\sigma'(\tau)$} or \mbox{$\sigma(\tau)=\u$} for all terms $\tau \in \Terms(\C \cup \F)$.
\qed
\end{definition}

Interpretations $\IAA$ have the same form as before, with \mbox{$I^h \subseteq I^t \subseteq \HB$} and $\sigma^h \preceq \sigma^t$, but under the extended definition of assignment and Herbrand Base.

\begin{definition}[$\cS$-satisfaction]\label{def:satisfy.s}
Given an interpretation
$\IAA$, we define when $\IA$ \ \emph{$\cS$-satisfies} a formula~$\varphi$ at some world $w\in \{h,t\}$, written $\IA,w \modelsS \varphi$ as follows:
\begin{enumerate}[ label=\roman*), leftmargin=20pt]
\item $\IA,w \modelsS p(\tau_1,\dots,\tau_n)$ if $p( \sigma^w(\tau_1), \dots, \sigma^w(\tau_n) ) \in I^w$ for any predicate $p/n \in {\cal P}$;
\label{item:p:modelS}
\item $\IA,w \modelsS \tau_1=\tau_2$ if $\sigma^w(\tau_1)=\sigma^w(\tau_2) \neq \u$;
\item $\IA,w \not\modelsS \bot$; $\IA,w \modelsS \top$;
\item $\IA,w \modelsS \alpha \wedge \beta$ if $\IA,w \modelsS \alpha$ and $\IA,w \modelsS \beta$;
\item $\IA,w \modelsS \alpha \vee \beta$ if $\IA,w \modelsS \alpha$ or $\IA,w \modelsS \beta$;
\item $\IA,w \modelsS \alpha \rightarrow \beta$ if for all $w'\geq w$: $\IA,w' \not\modelsS \alpha$ or $\IA,w' \modelsS \beta$;
\label{item:impl:modelS}


\item $\IA,w \modelsS \forall x \ \alpha(x)$ if for each $c \in \Uni$: $\IA,w \modelsS \alpha(c)$;
\label{item:forall:modelS}

\item $\IA,w \modelsS \exists x \ \alpha(x)$ if for some $c \in \Uni$: $I,w \modelsS \alpha(c)$.
\label{item:exists:modelS}

\end{enumerate}
As usual, we write $I \modelsS \varphi$, when $I,h \modelsS \varphi$.\qed
\end{definition}

It is easy to see that rules~\ref{item:p:modelS}-\ref{item:impl:modelS} for $\modelsS$ are the exactly the same as for $\modelsF$.
Rules~\ref{item:forall:modelS} and~\ref{item:exists:modelS} are just the result of replacing $\TermsF(\C)$ by $\Uni$.
From this observation, we can immediately establish a correspondence between $\modelsS$ and $\modelsF$ as follows.
Given any interpretation $\IAA$ for a signature \mbox{$\signature$},
by
\mbox{$\hat{\IA} = \tuple{\hat{\sigma}^h,\hat{\sigma}^t,I^h,I^t}$}
we denote a $\SQHTF$-interpretation for the signature
\mbox{$\Sigma = \tuple{\C \cup \cS,\F,\P }$}, for each $w \in \{h,t\}$,
where the assignment $\hat{\sigma}^w$ is the restriction of $\sigma^w$ to
\mbox{$\Terms^0(\C \cup \F) = \TermsF(\C \cup \cS \cup \F)$}.

\begin{Proposition}{\label{prop:SQHTF.correspondence.0}}
Any interpretation $\IA$ and $0$-formula~$\varphi$ satisfy:
\mbox{$\IA,\!w \modelsS \varphi$}
iff
\mbox{$\hat{\IA},\!w \modelsF \varphi$}.\qed
\end{Proposition}

As, in general, not all $\SQHTS$ formulas are $\SQHTF$ formulas, we cannot directly extend this result beyond \mbox{$0$-formulas}.
However, we may still expect that what can be proved to be a tautology using the $\SQHTF$ rules, still holds for $\SQHTS$ formulas.
For instance, the $\SQHTS$ formula
\mbox{$f = \set{\vec{\tau} : \varphi} \to f = \set{\vec{\tau} : \varphi}$}
is an obvious tautology which is not a $\SQHTF$ formula.
However, we may replace every occurrence of the intensional set~\mbox{$\set{\vec{\tau} : \varphi}$} by a fresh evaluable constant $c$ and observe that
\mbox{$f = c \to f = c$} is an $\SQHTF$ tautology, using this to conclude that it is an $\SQHTS$ tautology too.
To formalise this intuition,
let $\F_\varphi$ be a set disjoint from $\C \cup \F \cup \P$ containing a fresh constant per each different intensional set occurring in $\varphi$
and let $\kappa$ be a bijection mapping each intensional set in $\varphi$
to its corresponding constant in $\F_\varphi$.
Let us also denote by $\kappa(\varphi)$ the result of replacing in $\varphi$ each intensional set by its $\kappa$ image.
By
\mbox{$\tilde{\IA} = \tuple{\tilde{\sigma}^h,\tilde{\sigma}^t,I^h,I^t}$},
we denote the
$\SQHTF$-interpretation for the signature
\mbox{$\Sigma' = \tuple{\C \cup \cS,\F \cup \F_\varphi,\P}$}
where, for each $w \in \{h,t\}$, we have $\tilde{\sigma}^w(\tau) = \sigma^w(\kappa^{-1}(\tau))$ if $\tau \in \F_\varphi$ and $\tilde{\sigma}^w(\tau) = \sigma^w(\tau)$ otherwise.

\begin{Proposition}{\label{prop:SQHTF.correspondence}}
For any interpretation $\IA$ and formula~$\varphi$:
\mbox{$\IA,\!w \modelsS \varphi$}
iff
\mbox{$\tilde{\IA},\!w \modelsF \kappa(\varphi)$}.
\qed
\end{Proposition}

Note that, as illustrated by Proposition~\ref{prop:SQHTF.correspondence}, intensional sets act just as new fresh evaluable constants with respect to $\cS$-satisfaction.
To fix the expected meaning of each intensional set $\set{ \vec{\tau}(\vec{x}) \!:\! \varphi(\vec{x}) }$ we still have to relate its value in $\sigma^w$ to the satisfaction of formula $\varphi(\vec{x})$ in $\IA$.
Given $\tau=\set{ \vec{x} :  \vec{\tau}(\vec{x}) \!:\! \varphi(\vec{x}) }$, let us define its \emph{extension} at $\IA,w$ as:
\begin{IEEEeqnarray*}{l C l}
ext^0(\IA,w,\tau) & \eqdef & \{ \ \sigma^w(\vec{\tau}(\vec{c})) \mid \IA,w \modelsS  \varphi(\vec{c})
\,  \text{for some } \vec{c} \!\in\! \Uni^{|\vec{x}|} \ \}
\\
ext(\IA,w,\tau) & \eqdef & \begin{cases}
\u &\text{if } \u \in ext^0(\IA,w,\tau)
\\
ext^0(\IA,w,\tau) &\text{otherwise}
\end{cases}
\end{IEEEeqnarray*}

To put an example, if $\tau_1=\{X*n/X : p(X)\}$ and $\IA_1$ contains $I_1^t=\{p(0),p(1),p(2)\}$ then the set of tuples would be $\{0*n/0, \ 1*n/1, \ 2*n/2\}$.
Since the obtained set is finite, its evaluation coincides with the case of extensional sets in Def.~\ref{def:assign}, item~\ref{def:assign3}.

In the example, if we have, for instance, $\sigma^t(n)=10$, then $ext(\IA_1,t,\tau_1)=\sigma^t(\ \{0*10/0, \ 1*10/1, \ 2*10/2\}\ )=\{\u,10,10\}=\u$.
On the other hand, if
$I_1^t$ consists of the sequence $p(0),p(s(0)),p(s(s(0))), \dotsc$,
we similarly obtain the set $\{0*n/0, \ s(0)*n/s(0), \ s(s(0))*n/s(s(0)), \dotsc \}$ which, being infinite, is not covered by Def.~\ref{def:assign}.
With the definition of extension, for $\sigma^t(n)=10$, we get
$ext(\IA_1,t,\tau_1)=\{\sigma^t(0*10/0), \ \sigma^t(s(0)*10/s(0)), \ \sigma^t(s(s(0))*10/s(s(0))), \dotsc \}\ )=\{\u,10,10, \dotsc\}=\u$.

Now, given any interpretation \mbox{$\IAA$},
we define the assignments $\sigma_\IA^w$ for $w\in\{h,t\}$ as follows.
If $\tau \!\in\! \Terms^0(\C \!\cup\! \F)$ (i.e. not an intensional set) we can drop the $\IA$ subindex, that is $\sigma_\IA^w(\tau) \eqdef \sigma^w(\tau)$.
If $\tau$ is an intensional set
\begin{eqnarray*}
\sigma_\IA^t(\tau) & \eqdef &  ext(\IA,t,\tau)  \\
\sigma_\IA^h(\tau) & \eqdef & \left\{
\begin{array}{cl}
ext(\IA,h,\tau)
& \text{if } ext(\IA,h,\tau) = ext(\IA,t,\tau) 
\\
\tu & \text{otherwise }
\end{array}
\right.
\end{eqnarray*}
\normalsize
As we can see, we have two potential sources of undefinedness.
One may appear because some element in $ext(\IA_1,t,\tau)$ cannot be evaluated, as we had before with $0*n/0$.
But a second one may occur if the extension at $h$ is different from the one at $t$.
For instance, for the same example, the extension of $\tau_2=\{X:p(X)\}$ at $t$ is $\sigma^t(ext(\IA,t,\tau_2))=\{0,1,2\}$.
If we had $I_1^h=\{p(0),p(2)\}$, then the extension at $h$ would be $ext(\IA,h,\tau_2)=\{0,2\}\neq ext(\IA,t,\tau_2)=\{0,1,2\}$ and so $\sigma^h_\IA(\tau_2)=\u$.
We also define the interpretation $\Coherent{\IA} \eqdef \tuple{\sigma^h_\IA,\sigma^t_\IA,I^h,I^t}$.
Note that $\Coherent{\IA}$ is determined by the interpretation of predicates and terms in $\Terms^0(\C \cup \F)$.
In this sense, given a \mbox{$\SQHTF$-interpretation}~$\IA$ for the signature $\Sigma' = \tuple{\C \cup \cS, \F,\P}$, by $\Coherent{\IA}$, we also denote the interpretation $\Coherent{\JA}$ for any interpretation $\JA$
over the signature
\mbox{$\Sigma' = \tuple{\C, \F,\P}$} such that $\JA = \hat{\IA}$. 

\begin{definition}[Coherent interpretation]\label{def:coherent.set}
An interpretation $\IA = \tuple{\sigma^h,\sigma^t,I^h,I^t}$
is said to be \emph{coherent (w.r.t. intensional sets)} iff
$\IA = \Coherent{\IA}$.\qed
\end{definition}


\begin{definition}[Satisfaction]
We say that an interpretation $\IA$ \emph{satisfies (w.r.t. intensional sets)} a formula $\varphi$ at \mbox{$w \in \set{h,t}$}, in symbols \mbox{$\IA,w \models \varphi$},
if both $\IA$ is coherent and \mbox{$\IA,w \modelsS \varphi$}.
We also write 
\mbox{$\IA \models \varphi$}
when
\mbox{$\IA,h \models \varphi$}.
Given a theory $\Gamma$, we write $\IA \models \Gamma$ if $\IA$ is coherent and
$\IA \models \varphi$ for all formulas~$\varphi \in \Gamma$.
We say that a formula $\varphi$ is a \emph{tautology} if every coherent interpretation $\IA$ satisfies $\IA \models \varphi$.\qed
\end{definition}

\begin{Proposition}{\label{prop:SQHTF.correspondence.0.models}}
For any \mbox{$\SQHTF$-interp.} $\IA$ and $0$-formula~$\varphi$:
\mbox{$\IA \modelsF \varphi$}
iff
\mbox{$\Coherent{\IA} \models \varphi$}.\qed
\end{Proposition}

\begin{Proposition}{\label{prop:neg}}
Any coherent interpretation $\IA$ satisfies:
\begin{enumerate}[ topsep=1pt, itemsep=0pt, label=\roman*), leftmargin=15pt]
\item $\IA,w \models \varphi$ implies $\IA,t \models \varphi$, 
\item $\IA,w \models \neg \varphi$ iff $\IA,t \not\models \varphi$,\qed
\end{enumerate}
\end{Proposition}

\begin{Proposition}{\label{prop:tautologies}}
Given a formula $\varphi$, the following statements hold:
\begin{enumerate}[label=\roman*), leftmargin=20pt]
\item if $\kappa(\varphi)$ is an $\SQHTF$ tautology, then $\varphi$ is an $\SQHTS$ tautology,

\item if $\varphi$ is a $0$-formula, then $\varphi$ is an $\SQHTS$ tautology iff it is a $\SQHTF$ tautology.\qed
\end{enumerate}
\end{Proposition}

Nonmonotonicity is obtained with by the definition of equilibrium models for $\SQHTS$.

\begin{definition}[Equilibrium model]\label{def:equilibrium}
A total (coherent) model $\IA=\tuple{\sigma,I}$ of a theory $\Gamma$ is an \emph{equilibrium model} if there is no interpretation $\IA' \prec \IA$ which is also model of $\Gamma$.
A set of atoms $I$ is a \emph{stable model} of $\Gamma$ iff $\tuple{\sigma,I}$ is an equilibrium model of $\Gamma$ for some~$\sigma$.\qed
\end{definition}

\begin{Proposition}{\label{prop:SQHTF.conservative}}
Let $\Gamma$ be a theory just containing \mbox{$0$-formulas} over a signature~$\signature$ and let $I$ be a set of ground atoms. Then, $I$ is a stable model of $\Gamma$ according to Definition~\ref{def:equilibrium} iff
$I$ is a stable model of $\Gamma$ according to Definition~\ref{def:SQHTF.equilibrium}
with signature~$\Sigma' = \tuple{\C\cup \cS,\F,\P}$.\qed
\end{Proposition}

Proposition~\ref{prop:SQHTF.conservative} shows that our semantics is a conservative extension of~\cite{Cabalar11}.
Furthermore, as we will see later, our approach coincides with~\cite{gezh14a} and rejects \emph{vicious circles} as shown by the following example from~\cite{DovierPR03}.

\begin{example}\label{ex1}
Consider the following logic program $\newprogram\label{prg:dovier}$:
\begin{IEEEeqnarray*}{l *x+}
\begin{IEEEeqnarraybox}[][b]{l}
r(1).\\
r(2).
\end{IEEEeqnarraybox}
\hspace{1cm}
\begin{IEEEeqnarraybox}[][b]{l C l}
q(1).\\
q(2) &\lparrow& Z = \set{ X : r(X) } \wedge p(Z).
\end{IEEEeqnarraybox}
\hspace{1cm}
\begin{IEEEeqnarraybox}[][b]{l C l}
p(Y) &\lparrow& Y = \set{ X : q(X) }
\\
\end{IEEEeqnarraybox}
\hspace{1cm}
&\qed
\end{IEEEeqnarray*}
\end{example}
\noindent
$\program\ref{prg:dovier}$ has a unique equilibrium model $\tuple{\sigma_1,I_1}$
with $I_1 = \set{ q(1), p(\set{1}), r(1), r(2) }$
and
\begin{gather*}
\sigma_1( \set{ X : r(X) } ) \ = \ \set{1,2}
\hspace{2cm}
\sigma_1( \set{ X : q(X) } ) \ = \ \set{1}
\end{gather*}
Under~\cite{DovierPR03} semantics, there exists a second stable model $I_2$ = $\{q(1), q(2),$ $ p(\set{1,2}), r(1), r(2)\}$ not corresponding to any equilibrium model.
To see why, consider the coherent interpretation $\IA_2' = \tuple{\sigma_2^h,\sigma_2,I_2^h,I_2}$
with $I_2^h = I_2 \backslash \set{ q(2), p(\set{1,2}) }$ and:
\begin{IEEEeqnarray*}{lCl C lCl}
\sigma_2( \set{ X : r(X) } ) &=& \set{1,2}
\hspace{2cm}
\sigma_2( \set{ X : q(X) } ) &=& \set{1,2}
\\
\sigma_2^h( \set{ X : r(X) } ) &=& \set{1,2}
\hspace{2cm}
\sigma_2^h( \set{ X : q(X) } ) &=& \tu
\end{IEEEeqnarray*}
Note also that any total, coherent interpretation that agrees with $I_2$ on the true atoms must be of the form $\IA_2 = \tuple{\sigma_2,I_2}$ and that $\IA_2' \prec \IA_2$.
Hence, $\IA_2$ is not an equilibrium model.
Finally, note that $I_2$ violates the \emph{Vicious-Circle Principle}~\cite{gezh14a} because the truth of $q(2)$ depends on $p(\set{1,2})$ which, in its turn, depends on the fact that $\set{X : q(X) }$ contains element $2$. This last fact only holds if $q(2)$ holds.\qed

\section{Aggregates based on evaluable functions and intensional sets}
\label{sec:agg}

From now on,
we assume that $\F$ contains a subset $\A$ of function names of arity~$1$ used to denote aggregate names and that each aggregate name
\mbox{$f/1 \in \A$}
has an associated predefined
function
\mbox{$\hat{f} : \Ss \longrightarrow \C \cup \set{ \tu }$}
that computes its value as expected (maximum, count, sum, etc).
Now we further restrict Definition~\ref{def:coherent.set}
to fix the meaning of aggregates:

\begin{definition}\label{def:coherent.aggregates}
An interpretation $\IA = \tuple{\sigma^h,\sigma^t,I^h,I^t}$
is said to be \emph{coherent (w.r.t. aggregates)} if
it is coherent w.r.t. intensional sets
and, in addition,
it satisfies:
\begin{enumerate}[ topsep=3pt, itemsep=0pt, label=\roman*), leftmargin=20pt ]



\item for $f/n \in \A$, \
$\sigma^w(f(\tau)) = \hat{f}(\sigma^w(\tau))$ if $\sigma^w(\tau) \in \cS$; \ 
\mbox{$\sigma^t(f(\tau)) = \tu$} otherwise.

\end{enumerate}
We say that an interpretation $\IA$ is a \emph{model (w.r.t. aggregates)} of a formula $\varphi$, in symbols $\IA \models \varphi$, if both $\IA$ is coherent and $\IA \modelsS \varphi$.
Given a theory $\Gamma$, we write $\IA \models \Gamma$ if $\IA$ is coherent and
$\IA \modelsS \varphi$ for all formulas~$\varphi \in \Gamma$.\qed
\end{definition}

For the rest of the paper, we assume that the terms `coherent' and `model' are understood w.r.t. aggregates (Definition~\ref{def:coherent.aggregates}).
In particular, we assume that $\A$ contains at least
the aggregate names $\mathtt{count}$ and $\mathtt{sum}$
with the following semantics
\begin{enumerate}[ topsep=3pt, itemsep=0pt ]
\item $\widehat{\mathtt{count}}(S) = d$ if $d$ the number of elements in $S$,

\item $\widehat{\mathtt{sum}}(S) = d$ if
$S$ is a set of tuples, each of which has as an integer number as first component,
and $d$ is the sum of all first components of all tuples in $S$,
\end{enumerate}
\mbox{$\hat{f}(S) = \tu$} with
\mbox{$f \in \set{\mathtt{count}, \mathtt{sum}
}$}
otherwise.
We also assume that $\P$ and $\F$ respectively contain predicates $\leq,\geq,<,>,\neq$ and evaluable functions $+,-,\times,/$ for the arithmetic relations and functions with the standard meanings.
Similarly, $\P$ and $\F$ also contain the predicate $\in$ and the evaluable functions $\cup,\cap$ and $\backslash$ with the standard meanings in set theory.
As usual, we use infix notation for arithmetic and set predicates and functions.
We also omit the parentheses around intensional sets, so we write $\mathtt{count}\set{ X : p(X) }$
instead of 
$\mathtt{count}(\set{ X : p(X) })$.

\begin{example}\label{ex:1}
Let $\newprogram\label{prg:functional}$ be a theory 
over a signature with set of constants $\C = \set{a,b}$
that contains the rule~\eqref{f:count1} with $n=1$ plus
the fact $p(b)$.\qed
\end{example}

\normalfont
The theory in Example~\ref{ex:1} has no stable model.
On the one hand, it is clear that
every stable model $I$ must satisfy $p(b) \in I$.
Furthermore, $\set{ p(b) }$
is not a stable model
because, every coherent total interpretation $\tuple{\sigma,\set{ p(b) }}$ 
must satisfy
\mbox{$\sigma(\set{X:p(X)}) = \set{ b }$}
and
\mbox{$\sigma(\mathtt{count}\set{X:p(X)}) = 1$}.
Hence, 
$\tuple{\sigma,\emptyset}$ does not satisfy~\eqref{f:count1}.
On the other hand, the only other alternative is $\set{ p(a) ,p(b) }$
and we have that
$\IA = \tuple{\sigma,\set{ p(a) , p(b) }}$, with
\mbox{$\sigma(\set{X:p(X)}) = \set{ a , b }$}
and
\mbox{$\sigma(\mathtt{count}\set{X:p(X)}) = 2$},
satisfies~\eqref{f:count1}.
To show that $\IA$ is not an equilibrium model,
let us define
\mbox{$\IA' = \tuple{\sigma',\sigma,I',I}$} 
with
\mbox{$I' = \set{ p(b) }$}.
It is easy to see that $\IA' \models p(b)$.
Furthermore, we have that:
\begin{eqnarray}
ext(\IA',t,\set{X : p(X)}) \ = \ \set{a,b} & \neq & \set{b} \ = \ ext(\IA',h,\set{X : p(X)}) \label{f:vicious}
\end{eqnarray}
Since these two extensions are different,
it follows that \mbox{$\sigma'(\set{X : p(X)}) = \u$}
and, consequently, that
\mbox{$\IA' \not\models \mathtt{count}\set{X : p(X)} \geq 1$}.
In its turn, this implies that $\IA'$ is also a model of rule~\eqref{f:count1} with $n=1$ and a model of~$\program\ref{prg:functional}$.
Finally, it easy to check that $\IA' < \IA$ and, therefore, $\IA$ is not an equilibrium model.
As we will see in Section~\ref{sec:gz}, this behaviour agrees with $\Alog$~\cite{gezh14a}, but differs from other approaches like~\cite{sopo07a} and~\cite{ferraris11a} in which $\set{ p(a) , p(b) }$ is a stable model of~$\program\ref{prg:functional}$.

Interestingly, the use of evaluable functions allows defining aggregates within the logical language.
First, let us recall the notion of \emph{(directional) assignment} from~\cite{Cabalar11}.
By $f(\vec{\tau}) := \tau'$ we denote the implication\footnote{Note that $\tau'=\tau'$ can be read as ``$\tau'$ is defined.''} $(\tau'=\tau') \to f(\vec\tau) = \tau'$.
Then, rather than providing predefined $\widehat{\mathtt{max}}$ and $\widehat{\mathtt{min}}$ functions, we can specify their meaning as aggregates $\mathtt{max}$ and $\mathtt{min}$ by including, instead, the formulas:
\begin{IEEEeqnarray*}{lCl 'C' l}
\mathtt{max}(S) & := & X & \lparrow & X \in S \wedge \neg \exists Y (Y \in S \wedge Y>X)\\
\mathtt{min}(S) & := & X & \lparrow & X \in S \wedge \neg \exists Y (Y \in S \wedge Y<X)\end{IEEEeqnarray*}
Clearly, $\mathtt{max}(\emptyset)$ and $\mathtt{min}(\emptyset)$ are always left undefined, because no rule body can satisfy $X \in \emptyset$.
Similarly, $\mathtt{count}$ can be inductively defined in terms of addition as follows:

\begin{IEEEeqnarray}{lCl 'C' l}
\mathtt{count}(\emptyset) &:=& 0
\label{f:countrec1}\\
\mathtt{count}(S) &:=& 1 + \mathtt{count}(S\backslash \set{Y}) &\lparrow& Y \in S \label{f:countrec2}
\end{IEEEeqnarray}
That is, the cardinality of the empty set is $0$, and the cardinality of any other set is~$1$ plus the cardinality of any set obtained by removing one of its elements.\punctfootnote{For $\mathtt{count}$ and $\mathtt{sum}$, we are assuming that set $S$ is finite. Otherwise, we would need additional formalisation to deal with infinite sets and cardinalities.}
In general, we can easily define aggregate functions based on any
associative and commutative binary function.
For instance, to define the $\mathtt{sum}$ aggregate in terms of addition, we can just use:
\begin{IEEEeqnarray}{lCl 'C' l}
\mathtt{sum}(\emptyset) &:=& 0
\label{f:sumrec1}\\
\mathtt{sum}(S) &:=& \mathtt{sum}(S\backslash \set{Y}) + Y &\lparrow& Y \in S \label{f:sumrec2}
\end{IEEEeqnarray}
If we now consider a program~$\newprogram\label{prg:prod}$ containing these two rules, facts $p(2)$ and $p(3)$
plus
\begin{IEEEeqnarray}{l C l}
q(Y) &\leftarrow& \mathtt{sum}\set{X:p(X)} = Y
  \label{eq:prod.agg}
\end{IEEEeqnarray}
we can check that it has a unique stable model $I = \set{ p(2), p(3) , q (5) }$.
The stable model $I$ corresponds to the equilibrium model~$\tuple{\sigma,I}$
which satisfies the assignments:
\mbox{$\sigma(\mathtt{sum}(\set{2})) = 2$}, \
\mbox{$\sigma(\mathtt{sum}(\set{3})) = 3$} and
\mbox{$\sigma(\mathtt{sum}(\set{2,3})) = 5$}.

\section{Relation to GZ-aggregates for propositional formulas and $\Alog$}
\label{sec:gz}

A term $\tau$ is said to be \emph{arithmetic} if it only contains variables, numbers and arithmetic functions $+$, $-$, $\times$, etc.
A \emph{(GZ-)set name} is an intensional set of the form
\mbox{$\set{ \vec{x} : \varphi}$}
where $\vec{x}$ is a tuple of variables and $\varphi$ is a \mbox{$0$-formula}.
A $\emph{(GZ-)set atom}$ is an expression
of the form
$f(\tau) \rel \tau'$
with \mbox{$f \in \A$} an aggregate function, $\tau$ a set name,
$\tau'$ an arithmetic term and $\rel \in \set{=,\leq,\geq,<,>,\neq}$ an arithmetic relation.
A \emph{\mbox{GZ-predicate} atom} is an expression of the form $p(\tau_1,\dotsc,\tau_n)$
with $p/n \in \P$ a predicate name and $\tau_1,\dotsc,\tau_n \in \Uni$.
A \mbox{GZ-atom} is either a \mbox{GZ-predicate} atom or a set atom.
\GZformula s are the universal closure of formulas built over GZ-atoms using the connectives $\vee$, $\wedge$ and $\to$ as usual.
A \GZtheory\ is a set of \GZformula s.
We say that a \GZformula~$\varphi$ is \emph{ground} when there are no quantifiers and all arithmetic terms have been evaluated, that is,
the only variables occurring in $\varphi$ are bound to some set name
and the only arithmetic terms are numbers.
A \GZtheory\ is said to be \emph{ground} when all its formulas are ground.
The following definitions extend the semantics of $\Alog$~\cite{gezh14a} to arbitrary propositional formulas~\cite{CabalarFSS17}:

\begin{definition}\label{def:satisfy.cl}
A set of atoms $T$ \emph{satisfies} a ground GZ-formula $\fF$, denoted by $T\modelscl\fF$,~iff 
\begin{enumerate}[ topsep=1pt, itemsep=0pt, label=\roman*), leftmargin=15pt]
\item $T\not\modelscl\bot$
\item $T\modelscl p(\vec{c})$ iff $p(\vec{c})\in T$ for any ground atom $p(\vec{c})$

\item \mbox{$T\modelscl f\set{\vec{x}\!:\!\varphi(\vec{x})} \rel n$}
if
\mbox{$\hat{f}\big(\, \setm{ \vec{c} \in \Uni^{|\vec{x}|} }{ T\modelscl\varphi(\vec{c}) }\, \big)$}
has some value $k \in \integers$ and $k \rel n$ holds for the usual meaning of arithmetic relation $\rel$
\label{item:agg:def:satisfy.cl}


\item $T\modelscl\fF\wedge\fG$ iff $T\modelscl\fF$ and $T\modelscl\fG$
\item $T\modelscl\fF\vee\fG$ iff $T\modelscl\fF$ or $T\modelscl\fG$ 
\item $T\modelscl\fF \impl \fG$ iff $T\not\modelscl\fF$ or $T\modelscl\fG$.
\end{enumerate}
We say that $T$ is a \emph{model} of a ground GZ-theory $\Gamma$ if $T \modelscl \varphi$ for all $\varphi \in \Gamma$.\qed
\end{definition}

Given a formula $\varphi(\vec{x})$ with free variables $\vec{x}$,
by
\mbox{$\grnd(\varphi(\vec{x})) \eqdef \setm{ \varphi(\vec{c}) }{ \vec{c} \in \Uni^{|\vec{x}|} }$}
we denote the set of its ground instances.
By $\grnd(\Gamma) \eqdef \bigcup \setm{\grnd(\varphi(\vec{x}))}{\forall\vec{x}\varphi(\vec{x}) \in \Gamma}$ we also denote the grounding of a theory $\Gamma$.
Furthermore, given some set of atoms~$T$, we divide any theory $\Gamma$ into the two disjoint subsets: $\Gamma^+_T \ \eqdef \ \{ \fF \in \Gamma \mid T \modelscl \fF \}$ and $\Gamma^-_T \ \eqdef \ \Gamma \setminus \Gamma^+_T$, that is, the formulas in $\Gamma$ satisfied by~$T$ and not satisfied by~$T$, respectively.
When set $\Gamma$ is parametrized, say $\Gamma(z)$, we write $\Gamma^+_T(z)$ and $\Gamma^-_T(z)$ instead of $\Gamma(z)^+_T$ and $\Gamma(z)^-_T$.
We also omit the set brackets when the theory is a singleton.
For instance, $\grnd^+_T(\varphi)$ collects the formulas from $\grnd(\set{\varphi})$ satisfied by $T$.

\begin{definition}
\label{def:ferraris.reduct}
The reduct of a ground \GZformula~$\fF$
\wrt\ a set of atoms $T$
written $\fF^T$, is defined as
$\bot$ if $T \not\modelscl \fF$, otherwise:
\[
\fF^T \, \eqdef \, \left\{\!\!
\begin{array}{c@{\ }l}
a & \text{if } \fF \ \text{is some atom } a\in I\\

\big(\, \bigwedge \grndp{T}{\psi(\vec{x})}\, \big)^T 
  & \text{if } T \modelscl \varphi \text{ with } \varphi = f\set{\vec{x}\!:\!\psi(\vec{x})} \rel n \\

\fF_1^T \otimes \fF_2^T & \text{if} \ T \modelscl \fF \ \text{and } \fF=(\fF_1 \otimes \fF_2) \ \text{for some } \otimes \in \{\wedge,\vee,\to\}
\end{array}
\right.
\]
$T$ is a \emph{stable model} of a \GZtheory~$\Gamma$ 
iff $T$ is the $\subseteq$-minimal model of $\grnd(\Gamma)^T$.\qed
\end{definition}

\begin{Theorem}{\label{thm:gz}}
For any \GZtheory~$\Gamma$,
a set of atoms $T$ is a stable model of $\Gamma$ according to Definition~\ref{def:ferraris.reduct}
iff
$T$ is a stable model of $\Gamma$ according to Definition~\ref{def:equilibrium}.\qed
\end{Theorem}


Let us return to Example~\ref{ex:1} and recall 
we have shown that program~$\program\ref{prg:functional}$
has no stable model according to Definition~\ref{def:equilibrium}.
To illustrate the behaviour of this program with respect to $\Alog$ (Definition~\ref{def:ferraris.reduct}), note first that the only possible candidate for being a stable model is the set $I = \set{ p(a) , p(b) }$; otherwise, the rules would not be satisfied according to Definition~\ref{def:satisfy.cl}.
Then, we have that
$\program\ref{prg:functional}^I$ corresponds to the normal program
\begin{IEEEeqnarray*}{l C l}
p(a) &\lparrow& p(a) \wedge p(b)
\\
p(b)
\end{IEEEeqnarray*}
whose unique $\subseteq$-minimal model is $\set{ p(b) }$.
Hence, $I$ is not a stable model of~$\program\ref{prg:functional}$.
Intuitively, this is explained by the fact that our belief in~$p(a)$ depends on the extension of intensional set
\mbox{$\set{ X : p(X) }$} which, in its turn, depends on our belief in~$p(a)$, forming what~\citeN{gezh14a} call a ``vicious circle.'' According to the \emph{vicious circle principle}, set~$I$ should be rejected as a stable model.
In our approach, the ``vicious circle'' can be easily spotted by observing that evaluation of the set in world $h$ is left undefined because its extension is different from the one at world $t$, as shown in~\eqref{f:vicious}.

As an example of non-vicious definition, consider the following variation.

\begin{example}\label{ex:functional2}
Let $\newprogram\label{prg:functional2}$ be the following program: \normalfont
\begin{IEEEeqnarray}{lCl+x*}
p(a) & \leftarrow & \mathtt{count}\set{X:p(X) \wedge X\neq a}\geq 1 \label{f:count2}\\
p(b) & & & \notag \qed 
\end{IEEEeqnarray}
\end{example}

\noindent Again, the only candidate interpretation that satisfies all rules is $I=\set{ p(a) , p(b) }$, but the reduct corresponds now to:
\begin{IEEEeqnarray*}{l C l}
p(a) &\lparrow& p(b)
\\
p(b)
\end{IEEEeqnarray*}
whose unique minimal model is $I$, becoming a stable model under Definition~\ref{def:ferraris.reduct}.
Therefore, the same will happen under Definition~\ref{def:equilibrium}.
Let us put $\tau=\set{X:p(X)\wedge X\neq a}$.
It is easy to see that $\IA = \tuple{\sigma,I}$, with
\mbox{$\sigma(\tau) = \set{ b }$}
and
\mbox{$\sigma(\mathtt{count}(\tau)) = 1$},
satisfies~\eqref{f:count2} and obviously $p(b)$ as well.
Now take the smaller interpretation \mbox{$\IA' = \tuple{\sigma',\sigma,I',I}$} 
with \mbox{$I' = \set{ p(b) }$}.
Then, we have:
\begin{eqnarray*}
ext(\IA',t,\tau) \ = & \set{b} & = \ ext(\IA',h,\tau) 
\end{eqnarray*}
so \mbox{$\sigma'(\tau) = \set{ b }$}
and, consequently, 
\mbox{$\IA' \models \mathtt{count}(\tau) \geq 1$}.
In its turn, this implies that $\IA'$ does not satisfy~\eqref{f:count2} and so is not a model of~$\program\ref{prg:functional2}$.
It is easy to see that there is no other smaller interpretation $\IA'' < \IA$ that satisfies $p(b)$, and so $\IA$ is an equilibrium model.

An interesting property of $\Alog$ is that

it is always possible to introduce auxiliary variables for the aggregate value.
For instance, we can always replace \eqref{f:count1} by:
\begin{IEEEeqnarray}{l C l}
p(a) &\leftarrow& \mathtt{count}\set{X:p(X)} = Y \wedge Y \geq n
  \label{eq:prg:functional2}
\end{IEEEeqnarray}
This transformation is not safe in other semantics such as~\cite{sopo07a,fapfle11a,ferraris11a}.
In particular, under these semantics, if we take $n=0$, a program consisting of~\eqref{f:count1} has a unique stable model $\set{ p(a) }$
while a program consisting of~\eqref{eq:prg:functional2}
has no stable model.
These two programs are equivalent in $\Alog$ and have no stable model.
We can generalize the safety of this transformation to any context, using $\SQHTS$:

\begin{Proposition}[Existential variable introduction]{\label{prop:existencial.intr}}
Let $p(\tau_1,\dotsc,\tau_i,\dotsc,\tau_n)$ be an atom.
Then, $p(\tau_1,\dotsc,\tau_i,\dotsc,\tau_n)$ is equivalent in $\SQHTS$
to
\mbox{$\exists x [x = \tau_i \wedge p(\tau_1,\dotsc,x,\dotsc,\tau_n)]$}.\qed
\end{Proposition}

From 
Proposition~\ref{prop:existencial.intr},
it immediately follows that rule~\eqref{f:count1}
is equivalent to \eqref{eq:prg:functional2} with an existential quantifier $\exists Y$ in the body that can be trivially removed.

\section{Conclusions and Related Work}
\label{sec:conc}

We have extended Quantified Equilibrium Logic with evaluable, partial functions by introducing \emph{intensional sets}, that is, terms that allow defining elements in a set that satisfy some function or condition.
By providing a logical interpretation, we define the semantics of these new expressions without any kind of syntactic restriction, so they can be arbitrarily nested within standard logical terms and formulas.
The new extension yields a natural interpretation of an aggregate as an evaluable function applied to a set term.
As a result, the semantics of an aggregate can be fixed within the logical language, by the addition of formulas fixing its meaning, rather than relying on an external, predefined function (although we assume that some elementary set predicates are available).
This may become a powerful theoretical tool to analyse the fundamental properties of aggregate functions.
We have also proved that, when restricted to the syntactic fragment of language $\Alog$, our semantics coincides with that of~\cite{gezh14a}.

Extensions in Logic Programming with sets can be traced back to~\cite{Kuper90}
and~\cite{BeeriNST91}.
The approach of~\cite{DovierPR03} is based on the stable model semantics with a reduct definition, but does not include evaluable functions or allow complex terms (beyond simple variables) to appear as the definition of intensional sets.
Still, an important difference for the common syntactic fragment is that~\cite{DovierPR03} does not satisfy the \emph{vicious-circle free principle} as defined in~\cite{gezh14a} ``no object or property can be introduced by the definition referring to the totality of objects satisfying this property'' (see Example~\ref{ex1}).
The approaches \cite{Lee2009,ferrarisL2010} or \cite{HarrisonLR17} do not satisfy this principle either, but still share our orientation of defining the semantics of individual components of an aggregate.
An important difference, however, is that these formalisations do not deal with general \emph{evaluable} (i.e. non-Herbrand) functions, while we use them as a starting point and then understand aggregates just as one more case whose arguments happen to be sets.
This allows us a completely arbitrary use of aggregates as terms and of terms inside aggregates, leading to expressive constructions such as \eqref{f:countrec1}-\eqref{f:countrec2}.


Our future work will include the study of non-strict functions and their relation to~\cite{sopo07a,ferraris11a}.
We will also study the possible formalisation under Free Logics as in~\cite{CFPV14} or the kind of properties that allow functions to be recursively defined as in~\eqref{f:countrec1}-\eqref{f:countrec2}, and their application to~\cite{gezh14a}.

\newpage
\appendix

\section{Proofs of results}

\begin{Proofof}{\ref{prop:SQHTF.correspondence.0}}
Just note that, by construction, the evaluation of every $0$-term w.r.t. $\IAA$ is the same to its evaluation w.r.t. $\hat{\IA}$.
Hence, for every \mbox{$0$-terms} $\tau_1,\dotsc,\tau_n$
we have:\\
$\IA,\!w \modelsS p(\tau_1,\dotsc,\tau_n)$
\\iff 
$p(\sigma^w(\tau_1),\dotsc,\sigma^w(\tau_n)) \in I^w$
\\iff
$p(\hat{\sigma}^w(\tau_1),\dotsc,\hat{\sigma}^w(\tau_n)) \in I^w$
\\iff
$\hat{\IA},\!w \modelsF p(\tau_1,\dotsc,\tau_n)$.
Similarly, for any pair of \mbox{$0$-terms} $\tau_1,\tau_2$,
we have:\\
$\IA,\!w \modelsS \tau_1 =  \tau_2$
\\iff 
$\sigma(\tau_1) = \sigma(\tau_2)$
\\iff
$\hat{\sigma}(\tau_1) = \hat{\sigma}(\tau_2)$
\\iff
$\hat{\IA},\!w \modelsF \tau_1 = \tau_2$.
\\
Then, the proof follows by induction noting that the rules of $\modelsS$ and $\modelsF$ are the same when considered the different signatures.
\end{Proofof}

\begin{Proofof}{\ref{prop:SQHTF.correspondence}}
Just note that, by construction, the evaluation of every term w.r.t. $\IAA$ is the same to the evaluation of $\kappa(\tau)$ w.r.t. $\tilde{\IA}$.
Hence, for any terms $\tau_1,\dotsc,\tau_n$
we have:\\
\mbox{$\IA,\!w \modelsS p(\tau_1,\dotsc,\tau_n)$}
\\iff
\mbox{$p(\sigma(\tau_1),\dotsc,\sigma(\tau_n)) \in I^w$}
\\iff
\mbox{$p(\tilde{\sigma}(\kappa(\tau_1)),\dotsc,\tilde{\sigma}(\kappa(\tau_n))) \in I^w$}
\\iff
\mbox{$\hat{\IA},\!w \modelsF p(\kappa(\tau_1),\dotsc,\kappa(\tau_n))$}
\\iff
\mbox{$\hat{\IA},\!w \modelsF \kappa(p(\tau_1,\dotsc,\tau_n))$}.
\\Then, the proof follows by induction noting that the rules of $\modelsS$ and $\modelsF$ are the same when considered the different signatures.
\end{Proofof}

\begin{Proofof}{\ref{prop:SQHTF.correspondence.0.models}}
By definition, $\Coherent{\IA}$ is a coherent interpretation and, thus, we get:
\mbox{$\Coherent{\IA} \models \varphi$}
iff
\mbox{$\Coherent{\IA},\!h \modelsS \varphi$}.
Furthermore, by definition, $\IA$ and $\Coherent{\IA}$ agree on the evaluation of every $0$-term and, since $\varphi$ is a $0$-formula, it follows that
\mbox{$\Coherent{\IA},\!h \modelsS \varphi$}
iff
\mbox{$\JA,\!h\modelsS \varphi$}
for any interpretation $\JA$ such that $\JA = \hat{\IA}$.
Hence, the statement follows directly from Proposition~\ref{prop:SQHTF.correspondence.0}
\end{Proofof}

\begin{Proofof}{\ref{prop:neg}}
Let $\IAA$ be a coherent interpretation.
Then, we have that $\IA,\!w \models \varphi$ iff $\IA,\!w \modelsS \varphi$
and it is obvious that $\IA,\!w\modelsS \varphi$ implies $\IA,\!t \modelsS \varphi$ when $w = t$.
The proof that
$\IA,\!h\modelsS \varphi$ implies $\IA,\!t \modelsS \varphi$
easily follows by structural induction.
Note that, in case that $\varphi$ is an atom $p(\tau_1,\dotsc,\tau_n)$,
then
$\IA,\!h\modelsS \varphi$
implies 
$p(\tau_1,\dotsc,\tau_n) \in I^h \subseteq I^t$
which, in its turn, implies
$\IA,\!t\modelsS \varphi$.
In case that $\varphi$ is of the form $\tau_1 = \tau_2$,
we have 
$\IA,\!h\modelsS \varphi$
iff
$\sigma^h(\tau_1) = \sigma^h(\tau_2) \neq \tu$
which, in its turn, implies
$\sigma^t(\tau_1) = \sigma^t(\tau_2) \neq \tu$
and
$\IA,\!t\modelsS \varphi$.
The rest of the cases are as usual in $\SQHT$.
\\[10pt]
Let us show that
\mbox{$\IA,\!w \models \neg \varphi$}
iff
\mbox{$\IA,t \not\models \varphi$}.
Note that, since $\IA$ is coherent, we have:\\
\mbox{$\IA,w \models \neg \varphi$}
\\iff
\mbox{$\IA,w \modelsS \neg \varphi$}
\\iff
\mbox{$\tilde{\IA},w \modelsF \kappa(\neg \varphi)$} (Proposition~\ref{prop:SQHTF.correspondence})
\\iff
\mbox{$\tilde{\IA},w \modelsF \neg \kappa(\varphi)$} (by definition)
\\iff
\mbox{$\tilde{\IA},t \not\modelsF \kappa(\varphi)$} (Proposition~\ref{prop:SQHTF.neg}).
\\
Furthermore, since $\IA$ is coherent,
we have:\\
\mbox{$\IA,t \not\models \varphi$}
\\iff
\mbox{$\IA,t \not\modelsS \varphi$}
\\iff
\mbox{$\tilde{\IA},t \not\modelsF \kappa(\varphi)$}  (Proposition~\ref{prop:SQHTF.correspondence}).
\\
Consequently,
$\IA,w \models \neg \varphi$ iff $\IA,t \not\models \varphi$ holds.
\end{Proofof}

\begin{Proofof}{\ref{prop:tautologies}}
Assume first that $\varphi$ is a $\SQHTF$ tautology
and suppose, for the sake of contradiction, that
$\varphi$ is not a $\SQHTS$ tautology.
Let $\IAA$ be an interpretation such that $\IA \not\modelsS \varphi$.
Then, from Proposition~\ref{prop:SQHTF.correspondence},
it follows that $\IA \not\modelsF \kappa(\varphi)$ which is a contradiction.
Hence, $\kappa(\varphi)$ must be a $\SQHTS$ tautology.
\\[10pt]
Assume now that $\varphi$ is a $0$-formula.
Then $\kappa(\varphi) = \varphi$ and, as shown above, the only if direction holds.
Hence, assume that $\varphi$ is a $\SQHTS$ tautology and
suppose, for the sake of contradiction, that
$\varphi$ is not a $\SQHTF$ tautology.
Let $\IAA$ be an \mbox{$\SQHTF$-interpretation} such that $\IA \not\modelsF \varphi$.
From Proposition~\ref{prop:SQHTF.correspondence.0.models},
this implies that
$\Coherent{\IA} \not\modelsS \varphi$
which is a contradiction with the fact that $\varphi$ is a 
$\SQHTS$ tautology.
Consequently, $\varphi$ must be a $\SQHTS$ tautology.
\end{Proofof}

\begin{lemma}\label{lem:interp.coh.comp}
Any pair of \mbox{$\SQHT$-interpretations} $\IA_1$ and $\IA_2$ satisfy:
\begin{enumerate}[label=\roman*), leftmargin=20pt]
\item $\IA_1 \preceq \IA_2$ iff $\Coherent{\IA_1} \preceq \Coherent{\IA_2}$,
\label{item:1:lem:interp.coh.comp}

\item $\IA_1 = \IA_2$ iff $\Coherent{\IA_1} = \Coherent{\IA_2}$, and
\label{item:2:lem:interp.coh.comp}

\item $\IA_1 \prec \IA_2$ iff $\Coherent{\IA_1} \prec \Coherent{\IA_2}$.\qed
\label{item:3:lem:interp.coh.comp}
\end{enumerate}
\end{lemma}

\begin{proof}
First note that~\ref{item:1:lem:interp.coh.comp} implies~\ref{item:2:lem:interp.coh.comp} and these two together imply~\ref{item:3:lem:interp.coh.comp}.
Hence, let us show that~\ref{item:1:lem:interp.coh.comp} holds.
\\[10pt]
Let $\IA_1 = \tuple{ \sigma_1^h, \sigma_1^t,I_1^h,I_1^t}$ 
and
$\IA = \tuple{ \sigma_2^h, \sigma_2^t,I_2^h,I_2^t}$
such that $\IA_1 \preceq \IA_2$.
Then, $\sigma_1^w \preceq \sigma_2^w$ and $I_1^w \subseteq I_2^w$ with $w \in \set{h,t}$.
By definition, we have that
$\Coherent{\IA_1} = \tuple{ \sigma_{\IA_1}, \sigma_{\IA_1^t},I_1^h,I_1^t}$
and
$\Coherent{\IA_2} = \tuple{ \sigma_{\IA_2}, \sigma_{\IA_2^t},I_2^h,I_2^t}$
and, to show
$\Coherent{\IA_1} \preceq \Coherent{\JA_2}$,
it is enough to prove
$\sigma_{\JA_1} \preceq \sigma_{\JA_2}$ for $\JA \in \set{\IA, \IA^t}$.
Note that, for every term $\tau \in \Terms^0(\C\cup \F)$,
we have that
\begin{IEEEeqnarray*}{l ,C, l ,C, l ,C, l }
\sigma_{\IA}(\tau) &=& \sigma_1^h(\tau)
        &\preceq& \sigma_2^h(\tau) &=& \sigma_{\IA}(\tau)
\\
\sigma_{\IA^t}(\tau) &=& \sigma_1^t(\tau)
        &\preceq& \sigma_2^t(\tau) &=& \sigma_{\IA^t}(\tau)
\end{IEEEeqnarray*}
and, for every intensional set $\tau = \set{ \vec{\tau}(\vec{x}) \!:\! \varphi(\vec{x}) }$ we have that
\begin{IEEEeqnarray*}{l ,C, l ,C, l ,C, l }
\sigma_{\IA}(\tau) &\preceq& \sigma_{\IA}(\tau)
\\
\sigma_{\IA^t}(\tau) &\preceq& \sigma_{\IA^t}(\tau)
\end{IEEEeqnarray*}
follows from $I_1^w \subseteq I_2^w$.
The rest of the proof follows by structural induction and the fact that functions preserve their interpretation through subterms.
That is, $\tau = f(\tau_1,\dotsc,\tau_n)$
and $\sigma_{\JA_1}(\tau_i) \preceq \sigma_{\JA_2}(\tau_i)$.
By definition, if $\sigma_{\JA_1}(\tau_i) = \tu$ for some $1 \leq i \leq n$,
then
$\sigma_{\JA_1}(\tau) = \tu \preceq \sigma_{\JA_2}(\tau)$.
Otherwise,
$\sigma_{\JA_1}(\tau_i) = \sigma_{\JA_2}(\tau_i)$
for all $1 \leq i \leq n$
and, thus
\begin{align*}
\sigma_{\JA_1}(\tau)
  = \sigma_{\JA_1}(f(\sigma_{\JA_1}(\tau_1),\dotsc,\sigma_{\JA_1}(\tau_n)))
  = \sigma_{\JA_1}(f(\sigma_{\JA_2}(\tau_1),\dotsc,\sigma_{\JA_2}(\tau_n)))
  \preceq \sigma_{\JA_2}(\tau)
\end{align*}
and, by induction hypothesis, we get
\begin{align*}
\sigma_{\JA_1}(f(\sigma_{\JA_2}(\tau_1),\dotsc,\sigma_{\JA_2}(\tau_n)))
  \preceq \sigma_{\JA_2}(f(\sigma_{\JA_2}(\tau_1),\dotsc,\sigma_{\JA_2}(\tau_n)))
  = \sigma_{\JA_2}(\tau)
\end{align*}
Hence, $\sigma_{\JA_1}(\tau) \preceq \sigma_{\JA_2}(\tau)$
\end{proof}

\begin{Proofof}{\ref{prop:SQHTF.conservative}}
Assume first that $I$ is a stable model of $\Gamma$ w.r.t. Definition~\ref{def:equilibrium}.
Then, there is some total coherent interpretation
$\IA = \tuple{\sigma,I}$ such that $\IA \models \Gamma$
and that satisfies
\mbox{$\IA' \not\models \Gamma$}
for all $\IA'$ with
\mbox{$\IA' \prec \IA$}.
From
\mbox{$\IA \models \Gamma$},
it follows that
\mbox{$\hat{\IA} \modelsF \varphi$}
(Proposition~\ref{prop:SQHTF.correspondence.0}).
Suppose, for the sake of contradiction, that $I$ is not a stable model according to 
Definition~\ref{def:SQHTF.equilibrium}.
Then,
\mbox{$\hat{\IA} \modelsF \varphi$}
implies that there is some interpretation
\mbox{$\IA'  \prec \hat{\IA}$}
such that
$\IA' \modelsF \Gamma$.
From Proposition~\ref{prop:SQHTF.correspondence.0.models}, this implies that
$\Coherent{\IA'} \models \Gamma$.
Furthermore, from Lemma~\ref{lem:interp.coh.comp},
it follows that
\mbox{$\IA'  \prec \hat{\IA}$}
implies
\mbox{$\Coherent{\IA'} \prec \Coherent{\hat{\IA}} = \IA$}
which is a contradiction.
\\[10pt]
The other way around.
Assume now that $I$ is a stable model of $\Gamma$ w.r.t. Definition~\ref{def:SQHTF.equilibrium}.
Then, there is some interpretation
$\IA = \tuple{\sigma,I}$ such that
\mbox{$\IA \modelsF \Gamma$}
and that
\mbox{$\IA' \not\modelsF \Gamma$}
for all $\IA'$ with
\mbox{$\IA' \prec \IA$}.
From Proposition~\ref{prop:SQHTF.correspondence.0.models}, this implies that
\mbox{$\Coherent{\IA} \models \Gamma$}.
Suppose now that
$I$ is not a stable model according to 
Definition~\ref{def:equilibrium}.
Then,
there is some 
coherent interpretation $\IA' = \tuple{\sigma^h,\sigma^t,I^h,I} \prec \Coherent{\IA}$ such that
$\IA' \models \Gamma$.
From Proposition~\ref{prop:SQHTF.correspondence.0}, this implies that
$\hat{\IA}' \modelsF \Gamma$
and that $\hat{\IA}' \prec \IA$ which is a contradiction.
\end{Proofof}

\begin{Proposition}{\label{prop:modelscl}}
Given a ground \GZformula~$\varphi$
and a total coherent interpretation of the form
\mbox{$\IAT$},
we have:
\mbox{$\IA \models \varphi$}
iff
\mbox{$T \modelscl \varphi$}.\qed
\end{Proposition}

\begin{Proofof}{\ref{prop:modelscl}}
The proof follows by induction assuming $\varphi$ is an $i$-formula and that the statement holds for every subformula of $\varphi$ and for every
\mbox{$(i-1)$-formula}.
Note that~\ref{item:agg:def:satisfy.cl} is the unique non-trivial case.
\\[10pt]
Let \mbox{$A = (f\set{\vec{x}\!:\!\varphi(\vec{x})} \rel n)$} be a set atom.
Then, we have that\\
\mbox{$T\modelscl A$}
\\iff
\mbox{$\hat{f}\big(\, \setm{ \vec{c} \in \Uni^{|\vec{x}|} }{ T\modelscl\varphi(\vec{c}) }\, \big) = k$} and $k \rel n$
(Definition~\ref{def:satisfy.cl})
\\iff
\mbox{$\hat{f}\big(\, \setm{ \vec{c} \in \Uni^{|\vec{x}|} }{ \IA \models \varphi(\vec{c}) }\, \big) = k$} and $k \rel n$ (induction hypothesis).
\\
On the other hand,
we also have that\\
$\IA \models A$
\\iff
$\rel(\sigma(f\set{ \vec{x} : \varphi(\vec{x}) }),\sigma(n)) \in I^h$ (Definition~\ref{def:satisfy.s})
\\iff
$\sigma(f\set{ \vec{x} : \varphi(\vec{x}) }) \rel \sigma(n)$ 
\\iff
$\hat{f}(\sigma(\set{ \vec{x} : \varphi(\vec{x}) })) \rel \sigma(n)$ (Definition~\ref{def:coherent.aggregates})
\\iff
$\hat{f}(\setm{ \vec{x}[\vec{x}/\vec{c}]  }{ \IA \models \varphi(\vec{c}) \text{ with } \vec{c} \in \Uni^{|\vec{x}|} }) \rel \sigma(n)$ (Definition~\ref{def:coherent.set})
\\iff
$\hat{f}(\setm{ \vec{c} \in \Uni^{|\vec{x}|} }{ \IA \models \varphi(\vec{c}) }) \rel \sigma(n)$
\\iff
$\hat{f}(\setm{ \vec{c} \in \Uni^{|\vec{x}|} }{ \IA \models \varphi(\vec{c}) }) \rel n$ (term evaluation)
\\
Then, the result follows directly
by defining $k$ as the result of evaluating the expression
$\hat{f}(\setm{ \vec{c} \in \Uni^{|\vec{x}|} }{ \IA \models \varphi(\vec{c}) })$.
\end{Proofof}

\begin{Proposition}{\label{prop:reduct}}
Given a ground \GZformula~$\varphi$
and some coherent interpretation~\mbox{$\IA$},
we have:
\begin{enumerate}[ topsep=1pt, itemsep=0pt, label=\roman*), leftmargin=15pt]
\item \mbox{$\IA,\!t \models \varphi$} iff \mbox{$T \modelscl \varphi$}, and
\label{item:1:prop:reduct}

\item \mbox{$\IA \models \varphi$} iff \mbox{$H \modelscl \varphi^T$}.\qed
\label{item:2:prop:reduct}
\end{enumerate}
\end{Proposition}

\begin{Proofof}{\ref{prop:reduct}}
First, note that~\ref{item:1:prop:reduct}
follows directly from Proposition~\ref{prop:modelscl}.
So, let us prove~\ref{item:2:prop:reduct}.
\\[10pt]
Assume that $\IA$ is of the form $\IAH$.
If $\varphi$ is an ground \GZatom~$a$, then $\IA \models \varphi$
iff $a \in H \subseteq T$ iff $H \modelscl \varphi^T$.
Otherwise, we proceed by induction assuming
that $\varphi$ is an $i$-formula and that the statement holds for all subformulas of $\varphi$ and all
$(i-1)$-formulas.
\\[10pt]
If $\varphi$ is a set atom of the form
\mbox{$A=f\set{\vec{x}\!:\!\psi(\vec{x})} \rel n$}.
Then,
$\IA \models A$ implies 
\mbox{$\IA,t \models A$}
(Proposition~\ref{prop:neg})
which, in its turn, implies $T \modelscl A$ (Proposition~\ref{prop:modelscl})
and, thus, we get
$A^T = \big(\ \bigwedge \grndp{T}{\psi(\vec{x})}\ \big)^T$.
Furthermore, $\IA \models A$ also implies
\begin{align*}
\sigma^h(f\set{\vec{x}\!:\!\psi(\vec{x})}) &\neq \tu
&&&
\sigma^h(n) &= n \neq \tu
\end{align*}
By definition of term evaluation the former implies
\begin{align*}
\sigma^h(\set{\vec{x}\!:\!\psi(\vec{x})}) &\neq \tu
\end{align*}
and, by the definition of coherent interpretation, this implies 
\begin{align*}
\sigma^h(\set{\vec{x}\!:\!\psi(\vec{x})})
    \  &= \ \sigma^t(\set{\vec{x}\!:\!\psi(\vec{x})})
\\
    \ &= \  \setm{ \sigma^h(\vec{\tau}[\vec{x}/\vec{c}]) }
                          { \IA,h \modelsS  \psi(\vec{c})
  \text{ \ for some } \vec{c} \in \Uni^{|\vec{x}|} }
\\
\ &= \ 
  \setm{ \sigma^t(\vec{\tau}[\vec{x}/\vec{c}]) }{ \IA,t \modelsS  \psi(\vec{c})
                            \text{ \ for some } \vec{c} \in \Uni^{|\vec{x}|} }
\end{align*}
and, thus,
$\IA,h \models \psi(\vec{c})$
iff
$\IA,t \models \psi(\vec{c})$
iff
$T \modelscl \psi(\vec{c})$ (Proposition~\ref{prop:modelscl})
for all $\vec{c} \in \Uni^{|\vec{x}|}$.
This implies
$$\IA \models \bigwedge \grndp{T}{\psi(\vec{x})}
  = \bigwedge \setm{ \psi(\vec{c}) \in \grnd(T) }{  \text{ and } T \modelscl \psi(\vec{c}) }$$
Since this is a $(i-1)$-formula, by induction hypothesis, we get
$$H \modelscl \big(\ \bigwedge \grndp{T}{\psi(\vec{x})}\ \big)^T = A^T$$
Assume now that $H \modelscl A^T$,
then $T \modelscl A$ and we get
$$
\hat{f}\big(\, \setm{ \vec{c} \in \Uni^{|\vec{x}|} }{ T\modelscl\psi(\vec{c}) }\, \big) \rel n
$$
which
implies
$\hat{f}\big(\, \setm{ \vec{c} \in \Uni^{|\vec{x}|} }{ \IA,t \modelsS \psi(\vec{c}) }\, \big) \rel n$.
From this and Definition~\ref{def:coherent.set},
we get
$\hat{f}\big(\, \sigma^t(\set{ \vec{x} : \psi(\vec{x}) })\, \big) \rel n$
and, in its turn, from this and Definition~\ref{def:coherent.aggregates}
we get
$$
\sigma^t(f\set{ \vec{x} : \psi(\vec{x}) }) \rel \sigma^t(n)
$$
Hence, we obtain that $\IA,t \models A$.
Furthermore, $H \modelscl A^T =  \big(\ \bigwedge \grndp{T}{\psi(\vec{x})}\ \big)^T$
implies that, for all $\vec{c} \in \Uni^{|\vec{x}|}$,
$H \modelscl \psi(\vec{c})^T$ whenever $T \modelscl \psi(\vec{c})$.
By induction hypothesis,
this implies that
$\IA \models \psi(\vec{c})^T$ holds whenever
$T \modelscl \psi(\vec{c})$.
and, thus, we get that
$\IA,t \modelsS \psi(\vec{c})$
implies
$\IA \modelsS \psi(\vec{c})$
for all $\vec{c} \in \Uni^{|\vec{x}|}$.
Hence,
$$
\sigma^t(\set{\vec{x}\!:\!\psi(\vec{x})})
  \ \ = \ \ \sigma^h(\set{\vec{x}\!:\!\psi(\vec{x})})
$$
and, thus, that $\IA \models A$.
\\[10pt]
The cases for connective $\wedge$, $\vee$ and $\to$ follow by structural induction as in Lemma~1 from~\cite{ferraris05a}:
$\IA \models \varphi_1 \wedge \varphi_2$
(resp. $\IA \models \varphi_1 \wedge \varphi_2$)
\\iff
$\IA \models \varphi_1$ and (resp. or) $\IA \models \varphi_2$
\\iff
$H \modelscl \varphi_1^T$ and (resp. or) $H \modelscl \varphi_2^t$
\\iff
$H \modelscl \varphi_1^T \wedge \varphi_2^T$
(resp. $H \modelscl \varphi_1^T \vee \varphi_2^T$)
\\iff
$H \modelscl (\varphi_1 \wedge \varphi_2)^T$
(resp. $H \modelscl (\varphi_1 \vee \varphi_2)^T$).
\\[5pt]
Finally,
$\IA \models \varphi_1 \to \varphi_2$
\\iff both
$\IA,h \not\models \varphi_1$ or $\IA,h \models \varphi_2$
and
$\IA,t \not\models \varphi_1$ or $\IA,t \models \varphi_2$
\\iff both
$\IA \not\models \varphi_1$ or $\IA \modelscl \varphi_2$
and
$T \not\modelscl \varphi_1$ or $T \modelscl \varphi_2$
\\iff both
$H \not\modelscl \varphi_1^T$ or $H \modelscl \varphi_2^T$
and
$T \not\modelscl \varphi_1$ or $T \modelscl \varphi_2$
\\iff both
$H \modelscl \varphi_1^T \to \varphi_2^T$
and
$T \modelscl \varphi_1 \to \varphi_2$
\\iff both
$H \not\modelscl \varphi_1^T \to \varphi_2^T$
and
$\varphi^T = \varphi_1^T \to \varphi_2^T$
\\iff
$H \modelscl \varphi^T$
\end{Proofof}

\begin{lemma}\label{lem:grounding}
Let $\Gamma$ be any \GZtheory\ and let $\IA$ be any coherent interpretation
 and $T$ be a set of atoms.
Then, 
\begin{enumerate}[ topsep=1pt, itemsep=0pt, label=\roman*), leftmargin=15pt]
\item $\IA \models \Gamma$ iff $\IA \models \grnd(\Gamma)$,
\label{item:1:lem:grounding}

\item $T$ is a stable model of $\Gamma$ iff $T$ is a stable model of $\grnd(\Gamma)$.\qed
\label{item:2:lem:grounding}
\end{enumerate}

\end{lemma}

\begin{proof}
By definition, we get:
$\IA \models \Gamma$
\\iff
$\IA \models \forall\vec{x}\varphi(\vec{x})$ for all $\varphi(\vec{x}) \in \Gamma$
\\iff
$\IA \models \varphi(\vec{c})$ for all $\varphi(\vec{x}) \in \Gamma$ and all $\vec{c} \in \Uni^{|\vec{x}|}$
\\iff
$\IA \models \varphi(\vec{c})$ for all $\varphi(\vec{x}) \in \grnd(\Gamma) = \setm{ \varphi(\vec{c}) }{ \forall\vec{x}\varphi(\vec{x}) \in \Gamma \text{ and }\vec{c} \in \Uni^{|\vec{x}|}}$
\\iff
$\IA \models \grnd(\Gamma)$.
\\[10pt]
Furthermore, $T$ is a stable model of $\Gamma$
\\iff there is some total coherent interpretation $\IAT$ which is an equilibrium model of $\Gamma$
\\iff there is some total coherent interpretation $\IAT$ which is an $\prec$-minimal model of $\Gamma$
\\iff there is some total coherent interpretation $\IAT$ which is an $\prec$-minimal model of $\grnd(\Gamma)$
\\iff
$T$ is a stable model of $\Gamma$.
\end{proof}

\begin{Proofof}{\ref{thm:gz}}
First note that, from Definition~\ref{def:ferraris.reduct}
and Lemma~\ref{lem:grounding},
we have that $T$ is a stable model of $\Gamma$ iff $T$  is a stable model of $\grnd(\Gamma)$ according to both Definitions.
Hence, we assume without loss of generality that $\Gamma$ is ground.
\\[10pt]
Let $\IAT$ be a total coherent interpretation.
Then, from Proposition~\ref{prop:reduct}, 
we get that
$\IA \models \Gamma$ iff $T \modelscl \Gamma^T$.
Let us show now that if
$T$ is the $\subseteq$-minimal model of $\Gamma^T$, then
$\IA$ is an equilibrium model of~$\Gamma$.
Suppose, for the sake of contradiction, that this does not hold.
Then, there is a some coherent interpretation
\mbox{$\IA' = \tuple{\sigma^h,\sigma^t,H,T}$}
such that
\mbox{$\IA' \preceq \IA$}
and
\mbox{$\IA \models \Gamma$}, but
\mbox{$\IA' \not\succeq \IA$}.
Note that, from Proposition~\ref{prop:reduct}, 
\mbox{$\IA' \models \Gamma$}
implies
$H \modelscl \Gamma^T$
while,  since $\IA'$ is coherent, 
\mbox{$\IA' \preceq \IA$} and \mbox{$\IA' \not\succeq \IA$}
imply
$H \subset T$
(note that all evaluable functions are aggregates and, thus, $\sigma^h$ and $\sigma^t$ are fully determined by $H$ and $T$, respectively)
which is a contradiction with the assumption.

The other way around.
Suppose, for the sake of contradiction, that
$\IA$ is an equilibrium model of $\Gamma$,
but
$T$ is not the $\subseteq$-minimal model of $\Gamma^T$.
Then there is some set $H \subset T$ that satisfies $H \modelscl \Gamma^T$
and, from Proposition~\ref{prop:reduct}, this implies
\mbox{$\IA' = \tuple{\sigma^h,\sigma^t,H,T} \models \Gamma$}
and that $\IA' \prec \IA$
which contradicts the fact that $\IA$ is a stable model of $\Gamma$.
\end{Proofof}

\begin{Proofof}{\ref{prop:existencial.intr}}
Let $\IAA$ be some coherent interpretation.
If $\varphi$ is an atom, by definition, we have that
$\IA \modelsS \exists x, x = \tau_i \wedge p(\tau_1, \dotsc, x, \dotsc, \tau_n)$
\\iff
$\IA \modelsS c = \tau_i \wedge p(\tau_1, \dotsc, c, \dotsc, \tau_n)$
for some $c \in \Terms^0(\C)$
\\iff
$\IA \modelsS c = \tau_i$
and
$\IA \modelsS p(\tau_1, \dotsc, c, \dotsc, \tau_n)$
for some $c \in \Terms^0(\C)$
\\iff
$\sigma^h(c) = \sigma^h(\tau_i) \neq \tu$
and
$p(\sigma^h(\tau_1), \dotsc, \sigma^h(c), \dotsc, \sigma^h(\tau_n)) \in I^h$
for some constant \mbox{$c \in \Terms^0(\C)$}
\\iff
$\sigma^h(c) = \sigma^h(\tau_i) \neq \tu$
and
$p(\sigma^h(\tau_1), \dotsc, \sigma^h(\tau_i), \dotsc, \sigma^h(\tau_n)) \in I^h$
for some constant \mbox{$c \in \Terms^0(\C)$}
\\iff
$\sigma^h(\tau_i) \neq \tu$
and
$p(\sigma^h(\tau_1), \dotsc, \sigma^h(\tau_i), \dotsc, \sigma^h(\tau_n)) \in I^h$
\\iff
$p(\sigma^h(\tau_1), \dotsc, \sigma^h(\tau_i), \dotsc, \sigma^h(\tau_n)) \in I^h$
\\iff
$\IA \modelsS p(\tau_1, \dotsc, \tau_i, \dotsc, \tau_n)$.
\end{Proofof}
\end{document}